\documentclass[letterpaper]{article} 
\usepackage{aaai24}  
\usepackage{times}  
\usepackage{helvet}  
\usepackage{courier}  
\usepackage[hyphens]{url}  
\usepackage{graphicx} 
\usepackage{natbib}  
\usepackage{caption} 
\usepackage{algorithm}

\usepackage{newfloat}
\usepackage{listings}
\DeclareCaptionStyle{ruled}{labelfont=normalfont,labelsep=colon,strut=off} 
\floatstyle{ruled}
\newfloat{listing}{tb}{lst}{}
\floatname{listing}{Listing}
\usepackage{algorithm}
\usepackage[noend]{algpseudocode}
\usepackage{siunitx}
\usepackage{subcaption}
\usepackage{amsmath,amssymb,amsthm,bm}

\newtheorem{theorem}{Theorem}
\newtheorem{lemma}[theorem]{Lemma}
\newtheorem{definition}{Definition}%

\newcommand{\indep}{\perp \!\!\! \perp}
\newcommand{\ours}{\textsc{DCSurvival}}

\usepackage{color}
\newcommand{\chunkai}[1]{{\color{red}  [\text{Chunkai:} #1]}}
\newcommand{\weijia}[1]{[\text{Weijia:}]}

\title{Deep Copula-Based Survival Analysis for Dependent Censoring with Identifiability Guarantees}
\author{
    Weijia Zhang\textsuperscript{\rm 1\thanks{Corresponding author.}}
    Chun Kai Ling\textsuperscript{\rm 2}
    Xuanhui Zhang\textsuperscript{\rm 3}
}
\affiliations{
    \textsuperscript{\rm 1} School of Information and Physical Sciences, The University of Newcastle, Australia \\
    \textsuperscript{\rm 2} Carnegie Mellon University, USA\\
    \textsuperscript{\rm 3} School of Information Management, Nanjing University, China\\
    weijia.zhang.xh@gmail.com, chunkail@cs.cmu.edu, xuanhui@nju.edu.cn
}

\begin{document}

\maketitle

\begin{abstract}
Censoring is the central problem in survival analysis where either the time-to-event (for instance, death), or the time-to-censoring (such as loss of follow-up) is observed for each sample. The majority of existing machine learning-based survival analysis methods assume that survival is conditionally independent of censoring given a set of covariates; an assumption that cannot be verified since only marginal distributions are available from the data. The existence of dependent censoring, along with the inherent bias in current estimators has been demonstrated in a variety of applications, accentuating the need for a more nuanced approach. However, existing methods that adjust for dependent censoring require practitioners to specify the ground truth copula.
This requirement poses a significant challenge for practical applications, as model misspecification can lead to substantial bias. In this work, we propose a flexible deep learning-based survival analysis method that simultaneously accommodates for dependent censoring and eliminates the requirement for specifying the ground truth copula. 
We discuss the identifiability of our model under a broad family of copulas and survival distributions. Experiment results from a wide range of datasets demonstrate that our approach successfully discerns the underlying dependency structure and significantly reduces survival estimation bias when compared to existing methods.
\end{abstract}

\section{Introduction}
Survival analysis is a branch of statistical methods that focuses on modeling the time it takes for certain events to occur, with seminal work tracing back to mid-twentieth century such as the Kaplan-Meier estimator \cite{Kaplan1958} and Cox partial likelihood \cite{Cox1972}. 
Survival analysis has been widely applied in many disciplines, with applications in healthcare such as epidemiology \cite{Selvin2008}, clinical trials \cite{Emmerson2021}, and personalized medicine \cite{Zhang2017}, as well as equipment failure time analysis \cite{Voronov2018}. 

The most prominent challenge of survival analysis is the existence of \textit{censoring}, which occurs when the event time of a sample is not fully observed. 
Censoring is ubiquitous in clinical trials because a participant has the right to withdraw \cite{Ondrusek1998}. 
For example, in clinical follow-up study designed to evaluate the effect of radiology, the period of relapse-free survival can only be observed if a participant exhibits the expected symptoms during the span of the follow-up. 
However, the true time-to-event time remains unobservable if the participant withdraws prematurely or the study concludes prior to the cancer relapse. 
This inherent uncertainty requires algorithms that account for censoring. Neglecting censored observations can result in loss of efficiency and estimation bias unless the observations are missing completely at random \cite{Leung1997}.

A common assumption underpinning most machine learning and statistical survival analysis stipulates that censoring and survival are \emph{conditionally independent} given the observed covariates \cite{Wang2019}.
This assumption allows censored observations to be utilized by simultaneously maximizing the log-likelihood for both censored and uncensored samples.
Unfortunately, as only marginal distributions of event and censoring are available from data, the independent-censoring assumption cannot be verified from observational data \cite{Tsiatis1975} in the similar sense that the unconfoundedness assumption is unverifiable in causal inference \cite{Rubin1974}. 

In many real world applications, censoring mechanisms are in fact \emph{dependent} \cite{Kaplan1958,Leung1997,Templeton2020}.
For example, participants in clinical trials often prematurely remove themselves from the trial if they find the drug to be ineffective or experience adverse effects \cite{Scharfstein1999}, 
Similarly, in observational epidemiology, 
patients with a more advanced disease might be more likely to miss their follow-up clinic visits \cite{Howe2010}. 
Ignoring this dependence can result in biased survival estimations \cite{Kleinbaum2012}. 
One recent approach to account for the dependencies is to assume the parametric form of the dependencies and learn the correlation parameter \cite{Emura2018,Deresa2021,Uai2023}. 
These methods utilize \textit{copulas}, which are powerful statistical tools which model dependencies between random variables in isolation from their marginals. That is, if the copula between observed and censored times is known, then survival marginals may be unbiasedly estimated even in the presence of censored data. 
These methods face two significant and interrelated challenges. 
Firstly, specifying the parametric family of the underlying copula is inherently difficult, as practitioners may lack experience or prior information to make an informed choice.
Secondly, 
misspecified copulas will exacerbate model bias, leading to incorrect inference and misleading results.
More recent methods by \citeauthor{Deresa2021} alleviate these issues by \textit{learning the association parameters} of the copula. Nonetheless, the problem of misspecifying the copula's parametric form still exists.

In this paper, we tackle dependent censoring by eliminating the requirement for a pre-specified copula. 
We demonstrate that the copula characterizing the dependency structure can be identified under reasonably mild conditions and propose learning them using deep neural networks which are trained end-to-end alongside the marginals.
Specifically:
\begin{itemize}
	\item We propose the Deep Copula Survival ({\ours}) framework\footnote{ https://github.com/WeijiaZhang24/DCSurvival}, a deep copula-based survival analysis method that addresses \textit{dependent censoring} \textit{without requiring users to specify the parametric form of the }ground truth copula.
	\item We discuss the theoretical properties of our framework, demonstrating that under mild assumptions \textit{identifiability} is attainable with common parametric survival marginals and the Archimedean copula family.
	\item We evaluate our method on a variety of datasets, demonstrating that {\ours} successfully learns the underlying copula and \textit{significantly reduces bias} in survival predictions when compared to existing state-of-the-art. 
\end{itemize}

\section{Survival, Censoring and Copula}
We are working with a survival dataset $\mathcal{D}$ with the $i$-th sample denoted by $\mathcal{D}_i=(\bm{x}_i, t_i, \delta_i)$, where $\bm{x}_i \in \mathbb{R}^n$ is the $n$-dimensional covariates;
$t_i \in \mathbb{R}^+$ is the observed time; 
and $\delta_i \in \{ 0, 1\}$ is the event indicator.
We focus on the common scenario of \textit{right censoring}, where $t_i = \min(T_i, U_i)$ with $T_i,U_i\in\mathbb{R}^+$ denoting the latent event and censoring times respectively.
We have $\delta_i=1$ when the event time is observed, while $\delta_i=0$ when only the censored time is observed. 
We omit the subscript $i$ when the context is clear. Under this model, the likelihood of a survival data point $(\bm{x},t,\delta)$ under right-censoring is \cite{emura2018analysis}
\begin{align}
    \mathcal{L}
    =& \Pr(T=t, U > t |\bm{x})^{\delta} \Pr(T > t, U=t |\bm{x})^{1-\delta}
\label{equation:likelihood}
\end{align}
We denote the marginal distributions for event and censoring time by $S_{T|X}(t|\bm{x})=\Pr(T>t|\bm{x}) $ and $S_{U|X}(u|\bm{x})=\Pr(U>u|\bm{x})$, with density functions $f_{T|X}(t|\bm{x})=-\partial S_{T|X}(t|\bm{x}) /\partial t$ and $f_{U|X}(t|\bm{x}) = -\partial S_{U|X}(t|\bm{x}) /\partial t$.

\paragraph{Independent Censoring.} 

Most existing models assume that survival and censoring are independent, i.e., $T_i\indep U_i$; or conditionally independent given the covariates, i.e., $T_i\indep U_i \vert \bm{x}_i$.
The likelihood function in Equation \ref{equation:likelihood} simplifies to
\begin{align}
    \mathcal{L}_{\text{indep}} =& f_{T|X}(t|\bm{x}) S_{U|X}(t|\bm{x})^\delta
    \cdot f_{U|X}(t|\bm{x})S_{T|X}(t|\bm{x})^{1-\delta} \nonumber\\
    \propto &  f_{T|X}(t|\bm{x})^\delta \cdot S_{T|X}(t|\bm{x})^{1-\delta}.
\label{equation:likelihood_indep}
\end{align}
The density $f_{U|X}(t|\bm{x})$ and survival function $S_{U|X}(t|\bm{x})$ for the censoring distribution are often omitted during optimization as they are non-informative to the event densities and survival distributions \cite{Kleinbaum2012}. 
However, in observational studies, the censoring mechanism is not only often unknown to researchers, but also unidentifiable solely based on observational data \cite{Tsiatis1975}. 
This motivates methods to address dependent censoring without relying on specific or pre-determined assumptions about the censoring mechanism. One method is to utilize \textit{copulas}.

\paragraph{Copulas.}  Copulas model dependencies between random variables in isolation from their marginal distributions.
Loosely speaking, $\mathcal{C}(u_1,\cdots,u_d):[0,1]^d \to [0,1]$ is a $d$-dimensional copula if it is a distribution function of a random variable with support $[0, 1]^d$ and uniform marginals. 

Copula has found extensive applications thanks to Sklar's theorem \cite{Sklar1959}, which states that any $d$-dimensional continuous joint distribution can be \textit{uniquely} expressed with $d$ uniform marginals and a copula $\mathcal{C}$. More formally, 
\begin{theorem}
\cite{Sklar1959}. Let F be a distribution function with margins $F_1,\cdots,F_d$. There exists a $d$-dimensional copula $\mathcal{C}$ such that for any $(x_1,\cdots, x_d) \in \mathcal{R}^d$ we have $F(x_1,\cdots,x_d)=\mathcal{C}(F(x_1),\cdots,F(x_d))$. 
Furthermore, if the marginals $F_1,\cdots,F_d$ are continuous, $\mathcal{C}$ is unique.
\label{theorem:sklar}
\end{theorem}

In practice, \textit{Archimedean} copulas such as Clayton, Frank,  Gumbel, and Joe copulas are common. 
Archimedean copulas are defined based on $1$-dimensional generator $\varphi$, where
\begin{equation}
	\mathcal{C}(u_1, \cdots, u_d) = \varphi (\varphi^{-1}(u_1) + \cdots + \varphi^{-1}(u_d)).
\label{equation-generator-phi}
\end{equation}
Here $\varphi:[0,\infty] \to [0,1]$ is \emph{$d$-monotone}, i.e., $(-1)^k\varphi^{(k)}(u) \geq 0$ for $k\leq d $ and $u\geq 0$. We say that $\varphi$ is  \emph{completely monotone} if $(-1)^k\varphi^{(k)}(u)\geq 0$ for all $k\geq 0$.

\paragraph{Dependent Censoring via Copula.} 
Dependent censoring arises when unobserved confounders affect both survival and censoring times, leading to dependencies that must be accounted for when evaluating joint likelihoods in Equation \ref{equation:likelihood}. 
This is similar to confounding in causal inference (Figure~\ref{figure:illustration_maintext}): in dependent censoring, we never simultaneously observe the censoring and survival times for a subject; while in causal inference, we never observe the factual and counterfactual outcomes at the same time \cite{Pearl2009}.
\begin{figure}[t]
  \centering
  \begin{subfigure}[b]{0.45\linewidth}
    \includegraphics[width=\textwidth]{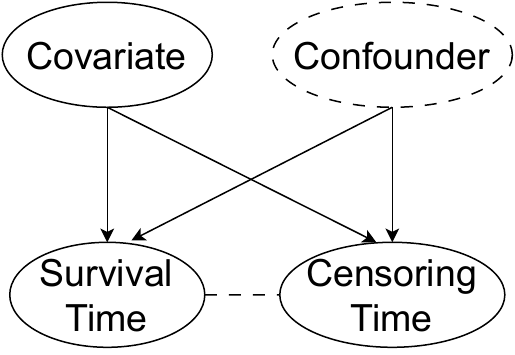}
    \caption{Dependent censoring.}
  \end{subfigure}
  \hfill 
  \begin{subfigure}[b]{0.45\linewidth}
    \includegraphics[width=\textwidth]{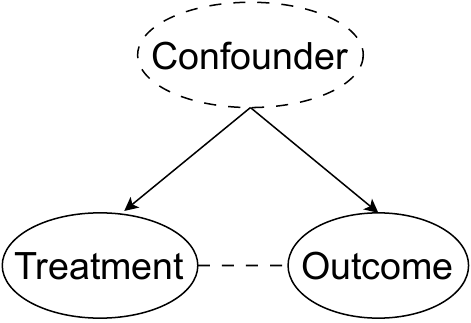}
    \caption{Unobserved confounding.}
  \end{subfigure}
  \caption{Illustrations of (a) dependent censoring in survival analysis and (b) unobserved confounding in causal inference. Solid/dashed nodes denote observed/hidden variables, respectively. The dashed lines between survival/censoring time and treatment/outcome indicate that estimation and evaluation in survival analysis with dependent censoring face similar challenges as in causal inference.}
  \label{figure:illustration_maintext}
\end{figure}

When survival and censoring times are dependent, applying Sklar's theorem yields the more general expression
\begin{align}
    \Pr(T > t, U > u \vert \bm{x}) = \mathcal{C}(S_{T|X}(t|\bm{x}), S_{U|X}(u|\bm{x})),
\label{equation:copula}
\end{align}
which when combined with Equation~\ref{equation:likelihood} yields the likelihood 
\begin{align}
    \mathcal{L}_{\text{dep}}=& \{ -\frac{\partial}{\partial u_1} \Pr(T>u_1, U>t |X=\bm{x} )\vert_{u_1=t}\}^{\delta} \nonumber\\ 
    & \cdot \{ -\frac{\partial}{\partial u_2} \Pr(T>t, U>u_2 |X=\bm{x}) \vert_{u_2=t} \}^{1-\delta}
    \nonumber \\
    = & \left\{ f_{T|X}(t|\bm{x}) \left. \frac{\partial}{\partial u_1} \mathcal{C}(u_1,u_2) \right\vert_{\substack{ u_1=S_{T|X}(t|\bm{x}) \\ u_2 = S_{U|X}(t|\bm{x}) }}\right\}^{\delta} \nonumber\\
    & \cdot \left\{ f_{U|X}(t|\bm{x}) \left. \frac{\partial}{\partial u_2} \mathcal{C}(u_1,u_2) \right\vert_{\substack{ u_1=S_{T|X}(t|\bm{x}) \\ u_2 = S_{U|X}(t|\bm{x}) }}\right\}^{1-\delta}.
\label{equation:likelihood_dep_main}
\end{align}
In this paper, we make the mild assumption that $\mathcal{C}$ does not depend on $\bm{x}$. We see that $\mathcal{L}_{\text{dep}}$ and $\mathcal{L}_{\text{indep}}$ are equivalent only when $\mathcal{C}$ is the independence copula $\mathcal{C}(u_1,u_2) = u_1 u_2$, which corresponds to the case where $T$ and $U$ are conditionally independent. 
For all other copulas, $\mathcal{L}_{\text{indep}}$ is biased because the dependency between $T$ and $U$ and the censoring marginals are not ignorable. 

To the best of our knowledge, all existing copula-based survival models require practitioners to \textit{specify the family of ground truth copula}. 
As most widely-used bivariate copula have closed-form partial derivatives, $\mathcal{L}_{\text{dep}}$ can then be optimized given the parameterization of the survival distributions and the assumed copula.
However, correctly defining the true copula presents an inherent challenge, placing practitioners in a difficult position as any misspecification can substantially amplify bias.

\section{End-to-end Survival Analysis via Copula}
We now introduce \ours, a framework seeking to learn both copula and survival distributions directly from right-censored data. In particular, we jointly maximize the likelihood in Equation~\ref{equation:likelihood_dep_main}, fitting parameters of marginal distributions of $T|X, U|X$ \textit{and} the parameters of the copula $\mathcal{C}$ between them. \ours \ optimizes over copula belonging to the Archimedean family; crucially, this includes the bivariate independence copula $C(u_1,u_2)=u_1 u_2$ and hence subsumes the independent-censoring assumption. 

\ours \ comprises two components, each based on the terms in Equation~\ref{equation:likelihood_dep_main}. The first is the Archimedean copula $\mathcal{C}$ between $T$ and $U$,  
the second comprises the density functions for event and censoring times $f_{T|X}$ and $f_{U|X}$. 
At a high level, suppose $\mathcal{C}(u_1, u_2), f_{T|X}$, and $f_{U|X}$ are parameterized by $\bm{\alpha}= (\theta, \theta_U, \theta_T)$ respectively, and that $\mathcal{C}$ is specified such that the partial derivatives $\partial \mathcal{C}(u_1, u_2)/ \partial u_1$ and $\partial \mathcal{C}(u_1, u_2)/ \partial u_2$ alongside their derivatives with respect to $\theta$ may be computed for all $u_1, u_2$. 
Then, one can compute $\mathcal{L}_{\text{dep}}$, and by applying the chain and product rules obtain the required gradients $\partial \mathcal{L}_{\text{dep}}/ \partial \bm{\alpha}$ needed to optimize $\bm{\alpha}$ in an end-to-end fashion via gradient descent.

Using this framework, one can restrict $U|X$, $T|X$ and $\mathcal{C}$ to be ``textbook'' distributions with few parameters to be estimated. This however leads to problems of model misspecification. 
\ours \ allows for the use of neural networks to model both copulas and margins, as long as they are ``fully differentiable'' in the manner described above.
We describe both components and how they relate to our problem of dependent censoring. 
For brevity, we defer specific architectural details to supplemental material.

\paragraph{Archimedean Copula for Event and Censoring Times.}
\citet{Kimberling1974} showed that the generator $\varphi$ being completely monotone is a necessary and sufficient condition for all \textit{two-dimensional} Archimedean copulas. 
Furthermore, utilizing the Bernstein-Widder characterization theorem \cite{Widder2010}, we know that every completely monotone function can be characterized using a mixture of negative exponential functions. Formally, we have:
\begin{theorem}
	(Bernstein-Widder) A generator $\varphi$ is completely monotone if and only if $\varphi$ is the Laplace transform of a positive random variable $M$, i.e., $\varphi(u)=\mathbb{E}_M [exp(-uM)]$ and $P(M>0)=1$.
\label{theorem-bernstein}
\end{theorem}
Combining Theorem \ref{theorem-bernstein} with the results of \citet{Kimberling1974} implies that any completely monotone generator $\varphi$ can be approximated using a \textit{finite} sum of negative exponentials \cite{Koyama2023}. 
This includes all bivariate Archimedean copulas thanks to the necessary condition.

As such, we learn the generator $\varphi$ (which implicitly defines the copula via Equation~\ref{equation-generator-phi} for $d=2$) via the method by \citet{Ling2020}, which uses neural networks to parameterize $\varphi$ with a large but finite mixture of negative exponentials. 
Roughly speaking, the neural network $\varphi_{nn}: \mathbb{R}^+ \rightarrow \mathbb{R}^+$ resembles a multilayer fully-connected network with a $1$-dimensional input, where each neuron's output is a convex combination of negative exponentials (of the input the network) with different rates. These are mixed with other neurons in the same layer (with some bias included) based on the network weights to give the outputs of the neurons of the next layer. As \citet{Ling2020} show, this rich architecture allows $\varphi_{\text{nn}}$ to satisfy Theorem~\ref{theorem-bernstein} and be a convex combination of an exponential (in the size of the network) number of negative exponentials.

Parameterizing an Archimedean copula via its generator has many benefits, chief of which is that 
we are able to compute $\frac{\partial}{\partial u_1} \mathcal{C}(u_1,u_2)$ and  $\frac{\partial}{\partial u_2} \mathcal{C}(u_1,u_2)$ evaluated at $u_1=S_{T|X}(t|\bm{x})$ and $u_2 = S_{U|X}(t|\bm{x})$ such that their derivatives with respect to the network parameters $\theta$ may in turn be used for gradient descent. In practice, this process is simplified via the use of automatic differentiation tools like Pytorch \cite{paszke2017automatic}. We note that computing these quantities will involve evaluating Equation~\ref{equation-generator-phi}, which necessitates estimating the inverse of $\varphi_{nn}$. This can be done efficiently via Newton's method since $\varphi_{nn}$ is one dimensional, and we refer to reader to \citet{Ling2020} for details.

\textit{Remark.} The network architecture of $\varphi_{nn}$ is not restricted to bivariate Archimedean copula and can be applied to survival analysis with competing risks. However, a key difference is that the completely monotone condition in Theorem \ref{theorem-bernstein} is not necessary for the generators for Archimedean copulas.
However, all \emph{extendible} Archimedean copulas can still be expressed with completely monotone generators when $d>2$ \cite{McNeil2009}. 
In this work, we will focus on dependent censoring in single-risk and leave research on competing risks for future endeavours.

\paragraph{Survival and Censoring Marginals. }
When it is reasonable to assume the parametric form of the survival and censoring distributions, the corresponding survival function $S_{T|X}(t|\bm{x})$ and the event density function $f_{T|X}(t|\bm{x})$ can be directly plugged into Equation \ref{equation:likelihood_dep_main} for evaluation. 
For example, CoxPH model with Weibull marginals \cite{Uai2023} or log-normal marginals can be used. In this case, identifiability is guaranteed: we discuss this in Section~\ref{sec:model-identifiability}.

When the parametric form for survival marginals is unknown, 
one can represent $S_{T|X}$ and $S_{U|X}$ using the monotonic neural density estimators (NDE) proposed by \citet{Chilinski2020}. The NDE network consists of a fully-connected covariate network for learning the representation of $X$, alongside a monotonic network parameterized with non-negative weights to ensure decreasingness with respect to $t$. 
Specifically, the output of the covariate network is concatenated with $t$ and fed to the layers with non-negative weights. 
The final layer of the monotonic network outputs a scalar with sigmoid activation yielding the survival function. 
Correspondingly, the density function may be expressed by $f_{T|X}(t|\bm{x})=-\partial S_{(T|X)}(t|\bm{x})/\partial t$ and $f_{U|X}(t|\bm{x})=-\partial S_{U|X}(t|\bm{x})/\partial t$, each computed via auto differentiation.

Once both marginals and copula are instantiated, their parameters are estimated by maximizing Equation \ref{equation:likelihood_dep_main} via stochastic gradient descent. 
Note that aside from survival distributions, \ours \ also estimates $\mathcal{C}$ as an added bonus. This may be of independent interest to practitioners.

\section{Model Identifiability}
\label{sec:model-identifiability}
Unidentifiable models hold no guarantee that the true parameters are recovered even with an infinite number of datapoints.
\citet{Tsiatis1975} showed that the joint distribution of the survival time $T$ and censoring time $C$ is not identifiable under a completely non-parametric setting. 
Therefore, we discuss identifiability under the mild assumption that the dependency can be characterized by the Archimedean family of copulas.
Furthermore, our approach also has the benefit of not making strict assumptions regarding the survival marginals. 
To our knowledge, existing previous work \cite{Uai2023} is not identifiable and requires practitioners to specify the exact form of the ground truth copula. 

Specifically, we consider identifiability for the following model with parameters $\boldsymbol{\alpha} = (\theta,\theta_T, \theta_U)$. 
In this section, we make the dependence on parameters more explicit (while omiting covariates $\bm{x}$ for brevity), yielding the following.
\begin{align}
&\Pr(T>t, U>u ) = \mathcal{C}_\theta (S_{T,\theta_T}(t), S_{U,\theta_U}(u)), \nonumber\\
&f_{y,\delta=1, \boldsymbol{\alpha}} (y) =  f_{T,\theta_T}(y) \left. \frac{\partial}{\partial u_1} \mathcal{C}_\theta(u_1,u_2) \right\vert_{\substack{ u_1=S_{T,\theta_T}(t) \\ u_2 = S_{U,\theta_U}(t) }} \nonumber\\
&f_{y,\delta=0,\boldsymbol{\alpha}} (y) = f_{U,\theta_U}(y) \left. \frac{\partial}{\partial u_2} \mathcal{C}_\theta(u_1,u_2)  \right \vert_{\substack{ u_1=S_{T,\theta_T}(t) \\ u_2 = S_{U,\theta_U}(t) }},
\label{equation-model}
\end{align}
where $f_{y,\delta=1, \boldsymbol{\alpha}} (y)$ and $f_{y,\delta=0, \boldsymbol{\alpha}} (y)$ are the densities of event and censored observations.

We start with the sufficient condition for identifying the model in Equation \ref{equation-model}. Specifically, we will show that the parameters $\boldsymbol{\alpha}=(\theta, \theta_U, \theta_T )$ can be uniquely determined from the observed dataset $\mathcal{D} = \{(\bm{x}_i,t_i,\delta_i)\}$. In other words, if $f_{y,\delta,\boldsymbol{\alpha}_1} (y)= f_{y,\delta,\boldsymbol{\alpha}_2} (y)$ for all $y$, then $\boldsymbol{\alpha}_1 = \boldsymbol{\alpha}_2$. 
\begin{theorem} \cite{Czado2022}
	Suppose the following conditions are satisfied:
(C1) For all $\theta_{T_1},\theta_{T_2} \in \Theta_T$ and $\theta_{U_1},\theta_{U_2} \in \Theta_U$, we have 
 \begin{align}
		&\lim\limits_{t\to 0} \frac{f_{T,\theta_{T_1}}(t)} {f_{T,\theta_{T_2}}(t)} = 1 \Longleftrightarrow \theta_{T_1} = \theta_{T_2}, \text{and} \nonumber\\
		&\lim\limits_{t\to 0} \frac{f_{U,\theta_{U_1}}(t)} {f_{U,\theta_{U_2}}(t)} = 1 \Longleftrightarrow \theta_{U_1} = \theta_{U_2}. 
       \label{equation:c1}
	\end{align}
	(C2) For alll $(\theta, \theta_T,\theta_U)\in\Theta\times\Theta_T\times\Theta_U$, we have
	\begin{align}
  &\lim_{t\to 0} \left. \frac{\partial}{\partial u_1} \mathcal{C}_{\theta}(u_1,u_2) \right\vert_{\substack{ u_1=S_{T,\theta_T}(t) \\ u_2 = S_{U,\theta_U}(t) }} = 1 \nonumber\\
  &\lim_{t\to 0} \left. \frac{\partial}{\partial u_2} \mathcal{C}_{\theta}(u_1,u_2) \right\vert_{\substack{ u_1=S_{T,\theta_T}(t) \\ u_2 = S_{U,\theta_U}(t) }} = 1
	\end{align}
	Then, the model defined in Equation \ref{equation-model} is identified.
	\label{theorem-identifiability}
\end{theorem}
\begin{proof}[Proof Sketch]
The proof consists of two main steps.
First, identification for the marginal distributions of the event and censoring times can be shown by tying the densities of observed event/censoring times to the true density of event and censoring distributions utilizing Condition (C2).
In a sense, this bears similarity to showing that the observed outcomes equal to the potential outcomes by expectation in causal inference. 
Second, the copula parameters can be identified by leveraging the unique existence of copulas.
\end{proof}

\begin{theorem} \cite{Czado2022}
Condition (C1) is satisfied by the families of Weibull, log-normal, log-logistic, and log-Student densities.
\label{theorem:marginals}
\end{theorem}

Theorem \ref{theorem-identifiability} and \ref{theorem:marginals} are attributed to \cite{Czado2022}. As it turns out, Archimedean copulas represented by $\varphi_{nn}$ have the attractive property of satisfying the condition (C2). 

\begin{figure*}
    \centering
    \includegraphics[width=\linewidth]{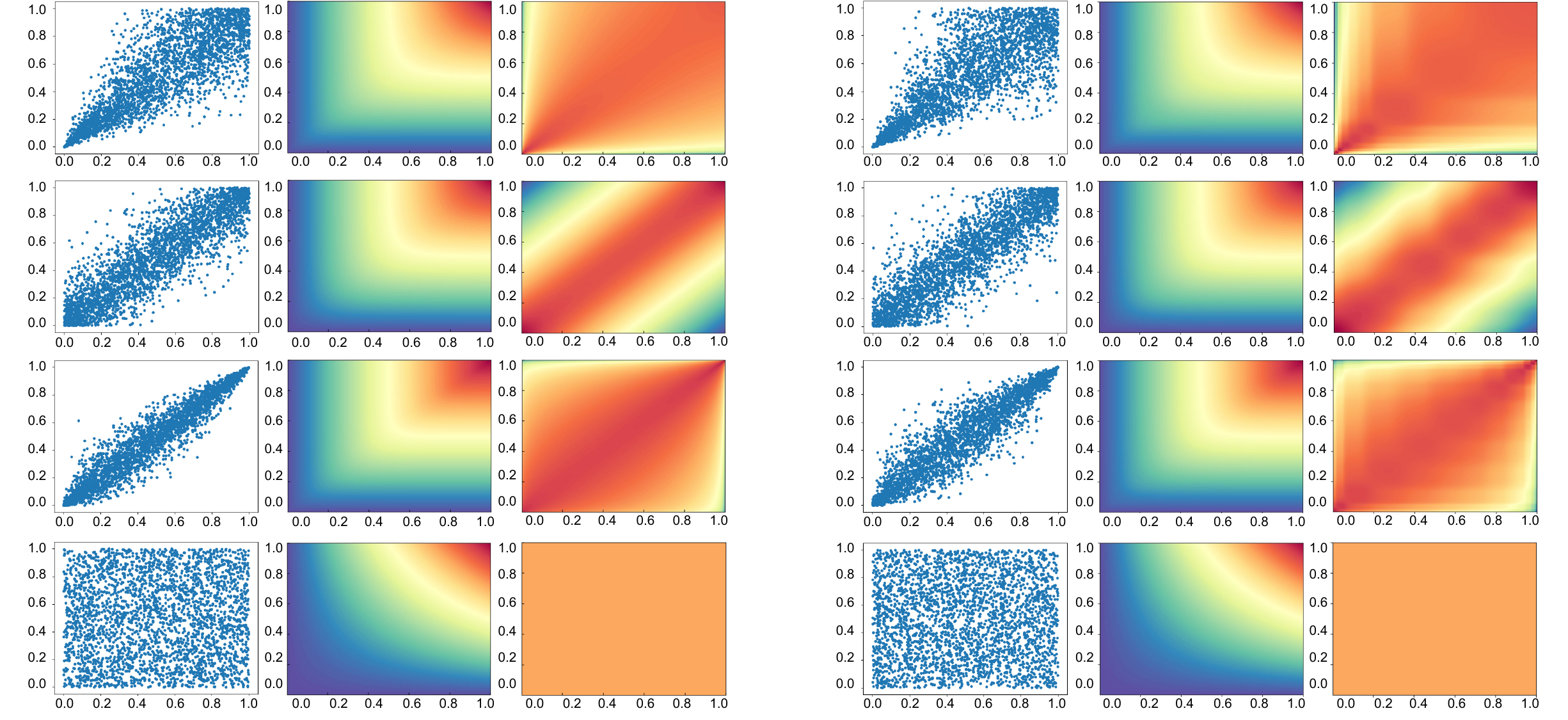}
        \begin{minipage}[t]{0.49\textwidth}
        \centering
        \subcaption{Ground Truth}
        \label{subfig:left}
    \end{minipage}
    \hfill
    \begin{minipage}[t]{0.49\textwidth}
        \centering
        \subcaption{Learned Copula}
        \label{subfig:right}
    \end{minipage}
    \caption{Top to bottom: learning Clayton, Frank, Gumbel, and Independence copulas using {\ours} from the \texttt{Linear-Rksk} dataset. Left to right: ground truth versus learned copula for the (i) scatter sample plots; (ii) joint cumulative distribution plots, (iii) log probability density plots. These figures are best viewed in colour.}
    \label{fig:synthetic_copula}
\end{figure*}

\begin{lemma}
Condition (C2) of Theorem \ref{theorem-identifiability} is satisfied if $\varphi$ is differentiable on $(0,\infty)$ and $\lim_{u\to 0} \varphi' (u) <0$.
\label{lemma-c2}
\end{lemma}
\begin{theorem}
Archimedean copulas represented by $\varphi_{nn}$ satisfy condition (C2) in Theorem \ref{theorem-identifiability}.
\end{theorem}

Therefore, \ours \ is identifiable when (i) $\mathcal{C}$ is defined using $\varphi_{nn}$ and (ii) $T$ and $U$ have margins belonging to the families in Theorem~\ref{theorem:marginals}. 
However, it is worth noting that Theorem 6 does not necessarily apply to all Archimedean copulas, but rather those that can be represented by $\varphi_{nn}$. In particular, \citet{Czado2022} have shown that not all Archimedean copulas are identifiable.
When using NDE to represent the survival marginals, the sufficient condition in (C1) of Theorem \ref{theorem-identifiability} is also satisfied since two neural networks with the same parameters produce the same outputs. 
However, the necessary condition is \textit{not} guaranteed.
Although the NDE instantiation does not strictly adhere to our identifiability condition, it compensates with superior flexibility,  requiring neither proportional hazard nor survival marginal assumptions.
Empirically, we find that {\ours} performs significantly better than state-of-the-art survival methods when dependent censoring is present, and achieves competitive performance under independent censoring,
further highlighting its robustness and versatility. 

\section{Experiments}
Evaluating dependent censoring methods is challenging due to the absence of a known censoring mechanism in observational datasets, akin to assessing treatment effect estimation in causal inference without ground truth causal effects \cite{Parikh2022}. 
To address this, we employ (i) synthetic datasets with defined dependency structures, (ii) semi-synthetic datasets using real-world covariates for censoring time simulation, and (iii) real-world datasets using metrics that do not need ground truth, as commonly done in causal effect studies \cite{Zhang2021}.


Experiments are conducted with one NVIDIA RTX4090 GPU.
We utilize Pytorch \cite{paszke2017automatic} for implementing all neural networks and automatic differentiation. 
Tensors are computed with double precision (fp64) as the inversion of $\varphi$ mandates numerical precision. 
When using Newton’s method to compute the inverse $\varphi_{nn}^{-1}$, we terminate when the error is less than $\num{1e-12}$. 
For all our experiments we set $\varphi_{nn}$ with $L = 2$ and $H_1 = H_2 = 10$, i.e., the copula representation contains two hidden layers with each of width $10$. 
The network is small but sufficiently expressive for the dependency structure since the generator $\varphi_{nn}$ is only $1$-dimensional. 
We use AdamW \cite{Loshchilov2018} for optimization and use 50\%/30\%/20\% training/validation/test splits. 
We used validation samples for early stopping based on the validation log-likelihood. No further hyperparameter tuning was performed. We provide our code and other details in the supplementary material.

\paragraph{Synthetic Datasets}
Following the approach of \citet{Uai2023}, we generate two synthetic datasets with the CoxPH model where the event and censoring risks are specified by Weibull marginal distributions \cite{bender2005generating}. 
Specifically, we sample two random variables from known copula and then apply inverse transform sampling to generate the event and censoring times $T_i$ and $U_i$ according to their Weibull distribution parameters. The observed $t_i$ is the minimum of the two.
We experiment with four Archimedean copulas, including the Clayton, Frank, Gumbel and Independence copulas and sample using copula-specific methods \cite{Scherer2012}.
We generate two datasets, \texttt{Linear-Risk} and \texttt{Nonlinear-Risk}, which correspond to cases where the the Weibull hazards are linear and non-linear functions of covariates (which are in turn drawn from $\mathcal{U}_{[0,1]}^{10}$).
Further details of the data generating procedure are provided in the supplementary materials.



\paragraph{Semi-Synthetic Datasets}
We use two semi-synthetic datasets based on the \texttt{STEEL} \cite{SathishkumarV2020} which contains 35,040 samples with 9 covariates, and the \texttt{Airfoil} \cite{ThomasBrooks1989} datasets which includes 1503 samples with 6 covariates. We induce censoring following a setting similar to those described in \cite{Uai2023}.
Briefly speaking, we use the dependent variable as the event time, and conditionally sample the censoring time with a copula. 
Contrary to the two synthetic datasets, the proportional hazard assumption is not maintained in the semi-synthetic datasets.
The semi-synthetic dataset generation method is similar to widely-used causal effect estimation benchmarks \cite{Hill2011,dorie2019}.

\paragraph{Real-World Datasets} 
We use two real-world datasets. 
The \texttt{SEER} dataset is from the Surveillance, Epidemiology and End Results database \cite{Howlader2010}.
Following \citet{Czado2022}, we use the monthly survival time of patients with localized pancreas cancer diagnosed between 2000 and 2015, and the event of interest is death caused by pancreas cancer. 
Patients that are alive or have died because of other cancers are considered as censored. There are 15 covariates that measure demographic and pathology features of 11,600 participants. However, as many common risk factors of diseases are not measured, the censoring and event time is likely to be correlated when patients are censored due to death caused by other diseases.

\texttt{GBSG2} is from the German Breast Cancer Study Group \cite{Schumacher1994}, which contains samples obtained from an observational study of 686 women with 8 covariates measuring the pathology characteristics of the participants.
The event of interest is the relapse-free survival time, while the participants are censored when they pass away. 
Both are likely positively correlated due to the correlations between cancer recurrence and patient death.
Further details of all datasets are provided in the supplementary material.

\begin{figure}[!t]
    \centering
    \includegraphics[width=\linewidth]{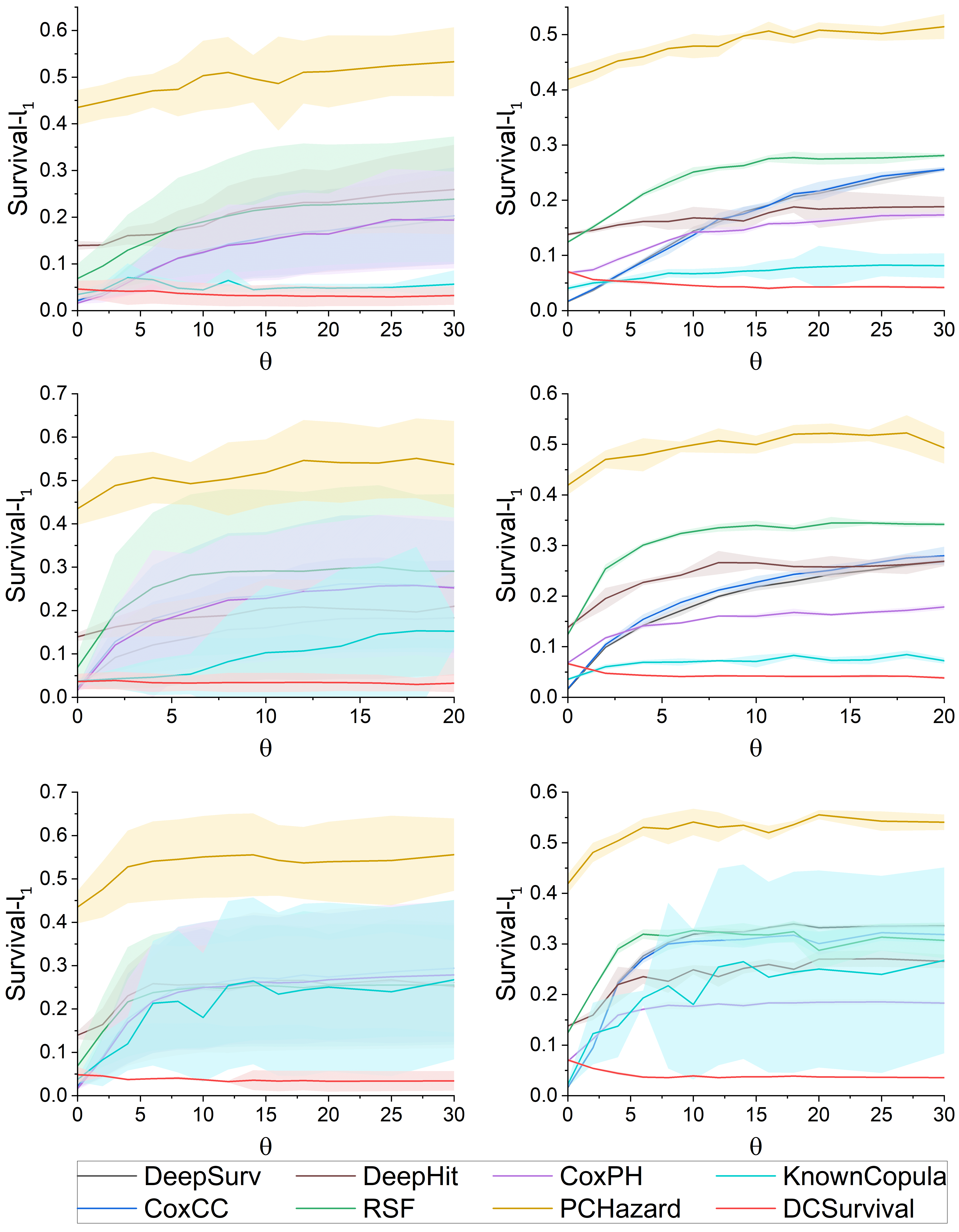} 
    \caption{Top to bottom (row): survival prediction biases of compared algorithms on varying censoring dependency governed by Frank, Clayton, and Gumbel copulas. 
    Left to right: \textsc{Linear-Risk} and \textsc{Nonlinear-Risk} results.
    The lines represent the Survival-$l_1$ means and the shaded areas are the standard deviations. Best viewed in colour.}
    \label{fig:survival_l1}
\end{figure}

\subsection{Identifying Censoring Dependency Structure}
We first empirically validate our identifiability results by instantiating {\ours} with parametric survival marginals that satisfy (C1).
Figure \ref{fig:synthetic_copula} shows the results of {\ours} for learning the copula from the \texttt{Linear-Risk} dataset by contrasting the sample scatter plots, joint cumulative distribution functions, and joint probability density functions of the ground truth copula with those learned by {\ours}. 
From Figure \ref{fig:synthetic_copula}, we can see that {\ours} is able to learn the dependency structure specified by Clayton, Frank, and Gumbel copulas from censored samples, and the contours of the log-likelihood almost exactly match the ground truth.
Additionally, {\ours} also correctly recovers the independent censoring mechanism when the ground truth is specified by the Independence copula. 
These empirical results corroborate the identifiability result in Theorem \ref{theorem-identifiability}, demonstrating that dependency structure can be learned from right-censored data. 
Experiments on \texttt{Nonlinear-Risk} dataset show similar good results which are provided in the supplementary material.


\begin{table}[t]
\centering
\begin{tabular}{ l|c|c}
\hline
  & \texttt{STEEL} & \texttt{Airfoil} \\
\hline
 \textsc{CoxPH} & $0.246\pm 0.008$ & $0.235\pm 0.054$ \\
 \textsc{DeepSurv} & $0.234\pm 0.007$ & $0.265\pm 0.031$ \\
  \textsc{CoxCC} & $0.265\pm 0.025$ & $0.264\pm0.100$\\
    \hline
  \textsc{RSF} & $0.249\pm 0.009$ & $0.474\pm0.034$ \\
  \textsc{DeepHit} & $0.157\pm 0.020$ & $0.259\pm0.055$ \\

 \textsc{PCHazard} & $0.212\pm 0.039$ & $0.200\pm0.038$\\
 \hline
 \textsc{KnownCopula} & $0.105\pm0.013$ & $0.215\pm0.060$ \\
 \hline
 {\ours} & $\bm{0.092\pm0.015}$ & $\bm{0.176\pm0.047}$ \\
\hline
\end{tabular}
\caption{The Survival-$l_1$ (mean$\pm$std) metrics of compared methods on test samples. 
The results reported are the averages obtained from repeating each experiment for 10 times.}
\label{table:semi-synthetic}
\end{table}

\subsection{Reducing Survival Estimation Bias}
We now assess the survival bias mitigation capabilities of {\ours} with the end-to-end NDE instantiation.
Under dependent censoring, frequently used survival metrics such as C-indices \cite{Harrell1982} and Brier scores \cite{Brier1950} are not proper scoring rules, i.e., the highest score is not achieved by the true distribution \cite{Uai2023}. 
Therefore, when the ground truths are known, we compare them with our estimated survival distributions 
using the \emph{Survival-}$l_1$ metric \cite{Uai2023}:
\begin{align}
    \resizebox{.88\linewidth}{!}{$\mathcal{C}_{\text{Survival}-l_1} = \sum\limits^{|\mathcal{D}|}_{i=1} \frac{1}{|\mathcal{D}| t^{max}_i} \int\limits_{0}^{\infty} |S(t|\bm{x}_i) - \hat{S}(t|\bm{x}_i)|dt,$}
\end{align}
where $S$ and $\hat{S}$ are the ground truth and estimated survival distributions, respectively.
For real-world datasets, we utilize \emph{calibration plots} \cite{NiculescuMizil2005} for evaluation, which can be obtained from the event indicators without accessing the true survival marginals.

We evaluate {\ours} against a diverse set of survival analysis techniques, including traditional methods \textsc{CoxPH} \cite{Cox1972} and Random Survival Forest (\textsc{RSF}) \cite{Ishwaran2008}, as well as deep learning-based approaches such as \textsc{DeepSurv} \cite{Katzman2018}, \textsc{DeepHit} \cite{Lee2018}, and \textsc{PCHazard} \cite{Kvamme2021}.
We also compare with a copula-based method that assumes dependent censoring \cite{Uai2023}, henceforth referred to as \textsc{KnownCopula}, as it requires known ground truth copula. 
Notably, \textsc{KnownCopula} also requires users to specify event and censoring marginals.

Results from Figure \ref{fig:survival_l1} and Table \ref{table:semi-synthetic} show that {\ours} consistently achieves the best performances under dependent censoring.
Observe that the performance of \textsc{KnownCopula} deteriorates as the dependency increases. This is possibly caused by the fact that explicit optimizing for $\theta$ involves calculating the exponentials of the inverse of $\theta$, posing substantial numerical challenges.
For independent censoring as shown in Figure \ref{fig:survival_l1}, {\ours} also performs competitively by consistently outperforming \textsc{RSF}, \textsc{DeepHit}, and \textsc{PCHazard}.
Importantly, algorithms surpassing {\ours} under independent censoring are restricted to the proportional hazard assumption which aligns with the data generation procedure.

In Figure \ref{fig:calibration}, calibration plots for the \texttt{SEER} and \texttt{GBSG2} datasets demonstrate that \textsc{DCSurvival} is the closest to the ideal calibration represented by the black dashed line. 
Calibration plot assesses whether an algorithm underestimate/overestimate event risks by comparing the average predicted event risk with the overall event rate \cite{Calster2019}.
Other compared methods systematically underestimate the event probability, as evidenced by their plots residing substantially above the ideal calibration.
For \textsc{KnownCopula}, because the ground truth copula is unknown, we use a convex combination of Frank and Clayton copulas, and simultaneously optimizing for the mixing weight and copula parameters as discussed in \cite{Uai2023}. However, this approach does not yield good performance, possibly due to misspecification of both the copula and survival marginals.
These observations emphasize the value of data-driven methods like  {\ours}, which jointly learns copula and survival marginals.


\section{Related Work}
\paragraph{Independent Censoring Methods.}
Most statistical survival models assume independent censoring.
Notably, the partial likelihood function introduced in the Cox Proportional Hazard (CoxPH) model \cite{Cox1972} explicitly relies on the independent censoring assumption \cite{Jackson2014}.
Furthermore, classical models such as the Kaplan-Meier estimator \cite{Kaplan1958} and Nelson-Aalen estimator \cite{Nelsen1998} also assume independent censoring.
For machine learning-based methods, 
Random Survival Forest \cite{Ishwaran2008} ensembles survival trees build with the log-rank test splitting criterion which assumes independent censoring.
Neural network-based approaches, such as Fraggi-Simon Net \cite{Faraggi1995}, DeepSurv \cite{Katzman2018}, and CoxCC\cite{Kvamme2019} adhere to the proportional hazard assumption of the CoxPH model, and thus also mandate independence.
DeepHit \cite{Lee2018} discretizes the event time horizon and accommodates competing risks with event-specific neural networks. 
Recently, several tailored neural models such as differential equation networks \cite{Tang2022} and monotonic networks \cite{Rindt2022} have also been proposed for survival data. Unfortunately, they all assume independent or conditional independent censoring.


\begin{figure}[!t]
    \centering
    \includegraphics[width=\linewidth]{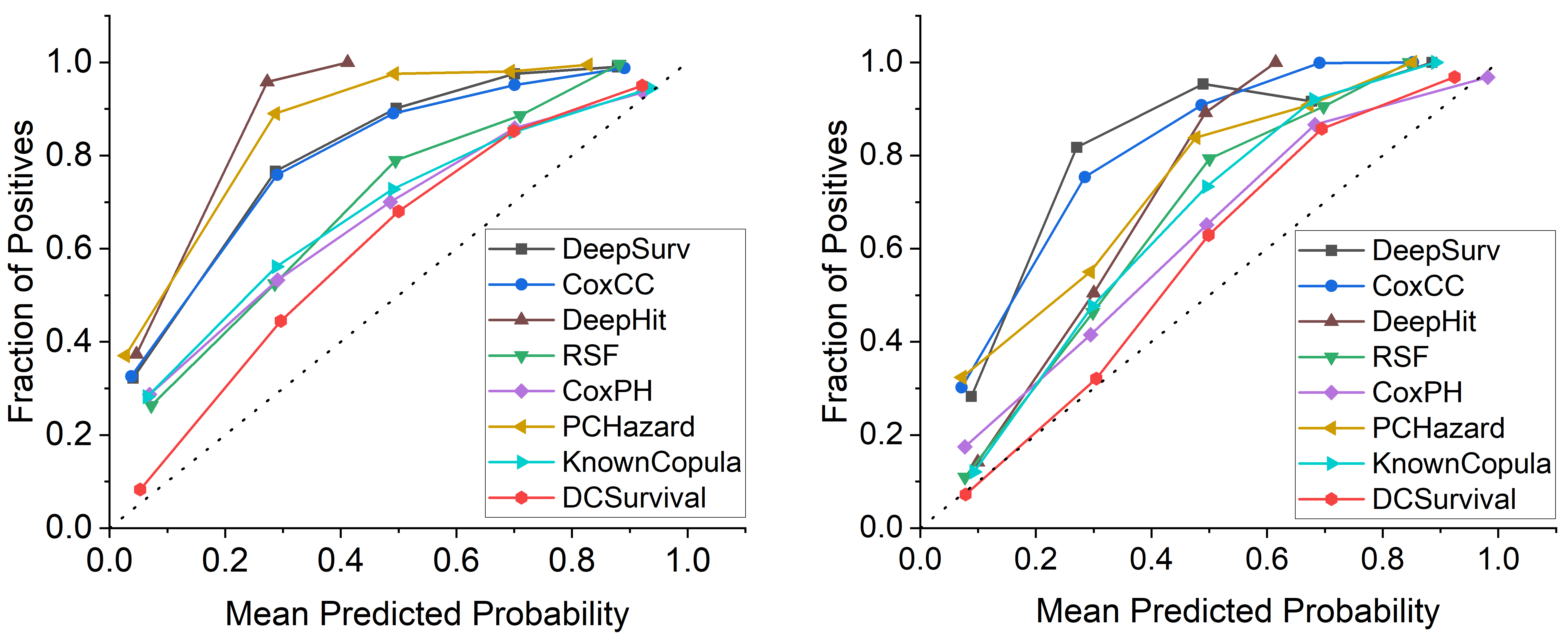} 
    \caption{Calibration plots of on test samples of the \texttt{SEER} (left) and \texttt{GBSG2} (right) datasets. The plots of better-calibrated algorithms are closer to perfect calibration (dashed black line). Best viewed in colour.}
    \label{fig:calibration}
\end{figure}

\paragraph{Dependent Censoring Methods.}
Discussion on dependent censoring in survival analysis can be traced back to the seminal work of \citet{Tsiatis1975} and \citet{Lagakos1979}.
To account for dependent censoring, several copula-based approaches have been explored by assuming the parametric family of the ground truth copula is known. Specifically, \citet{Czado2022,Deresa2021,Deresa2022} proposed statistical methods for parametric and semiparametric survival marginals with identifiability guarantee for all parameters.   
\citet{Midtfjord2022} and \citet{Uai2023} proposed to use boosting and neural network for maximizing the log-likelihood with pre-specified copulas, respectively. 
To the best of our knowledge, all existing copula-based approaches require practitioners to provide the ground truth copula, or at least its parametric form.
Our proposed {\ours} algorithm relaxes this assumption by approximating the copula with neural networks and learning their parameters.

\section{Conclusion}
In this paper, we propose {\ours}, a deep copula-based survival algorithm that does not require user-specified ground truth copula. 
{\ours} is capable of survival modelling under a wide range of censoring dependencies described by the Archimedean family of copulas.
Theoretically, we show that the parameters for both the copula and the survival distributions are identifiable under mild assumptions.
Empirically, we demonstrate that {\ours} successfully recovers the censoring dependencies and significantly reduces survival estimation bias.
Future work include moving beyond Archimedean copulas and developing neural survival marginal estimators that are identifiable.

\section*{Acknowledgments}
The authors would like to thank Prof. Ingrid Van Keilegom for her discussion and constructive comments on this work. 

Xuanhui Zhang is supported by the National Science Foundation of China (72204110).
\bibliography{aaai-23}

\newpage
\appendix
\onecolumn 
	
\section{Survival Marginals}
One of the primary goals in survival analysis is to estimate the survival function. The survival function $S$, defined as:
\begin{equation}
	S(t) = Pr (T>t),
\end{equation}
is a monotonically decreasing function characterizing the probability of subject survival past time $t$. 
A probability quantity closely related to the survival function is the distribution function $F(t) = 1-S(t)$, which represents the probability of an event occurs before $t$. 

Another desirable estimand of survival analysis is the hazard function $h(t)$. Intuitively, the hazard function is the instaneous failure rate conditional on the subject has not failed until time $t$. Specifically, $h(t)$ is defined as
\begin{equation}
	h(t) = \lim\limits_{\Delta t\to 0}\frac{Pr(t \leq T < t + \Delta t \vert T \geq t)}{\Delta t} = \frac{f(t)}{S(t)},
\end{equation}
where $f(t)= \frac{\partial}{\partial t}F(t)= -\frac{\partial}{\partial t} S(t) $ denote the death density function. Furthermore, it is straightforward to represent the hazard function $h(t)$ as:
\begin{equation}
	h(t) = \frac{f(t)}{S(t)} = -\frac{d}{dt}S(t) \cdot \frac{1}{S(t)} = -\frac{d}{dt}[\log S(t)].
\end{equation}

The cumulative hazard function is defined as the integral of the hazard function:
\begin{align}
    H(t) = \int_0^t h(u)du=-\log S(t).
\end{align}
Therefore, the survival function can be also expressed as $S(t)=\exp(-H(t))$.

\section{Censoring Mechanisms}
We focus our discussion of censoring mechanisms on right-censoring.
The censoring mechanism in most observational studies are unknown with few exceptions \cite{Leung1997}.
Specifically, there are two types of known censoring mechanism, namely Type I and Type II censoring.
Type I censoring occurs in clinical observational studies when every subject is followed up for a specified time of $C_0$ or until event time. 
Type II censoring is often used in engineering experiments where a total of $n$ devices are used in the study but, instead of continuing until all devices fail, the study is terminated when a total of pre-defined $r$ devices fail. 
Generally speaking, Type II censoring does not affect estimation because it does not affect the likelihood model. However, Type II censoring is rarely adopted in applications except engineering.
Both type I and type II censoring happen due to \textit{end of study}. 

However, as conducting trials and experiments are often expensive and time consuming, data with known censoring mechanism is rather difficult to collect.
Therefore, researchers often have to be content with observational studies where the data is passively collected and censoring happens due to \textit{lost of follow up}. 

Therefore, in most applications that use observational data, the censoring mechanism are unknown and more complicated than the Type I and Type II designs. 
Within statistical literature, there are several widely used assumptions \cite{Kleinbaum2012} (we use intuitive definitions for readability, and refer readers to the referenced textbook for mathematical definitions) :

\begin{definition}
    (\emph{Random Censoring}) Random censoring indicates that subjects who are censored at time $t$ are representative of all subjects who remain at risk at time $t$ with respect to the event of interest.
\end{definition}

\begin{definition}
    (\emph{Independent Censoring}) Independent censoring indicates that within any subgroup of interest, subjects who are censored at time $t$ are representative of all subjects in that subgroup who remain at risk at time $t$ with respect to the event of interest. 
\end{definition}


\begin{definition}
    (\emph{Non-informative censoring}) Non-informative censoring indicates the distribution of survival time $T$ provides no information about the distribution of censoring time $U$, and vice versa.
\end{definition}


Strictly speaking, independent censoring assumptions is a special case of non-informative censoring \cite{Lagakos1979}. However, since examples where censoring is non-informative but not independent are artificially crafted\cite{Williams1977}, we do not differentiate these two assumptions in our discussions.

Despite the importance of non-informative censoring assumption, it is impossible to identify whether the censoring mechanism is non-informative from observational data \cite{Tsiatis1975}. Furthermore, as Kaplan and Meier \cite{Kaplan1958} noted, ``in practice this assumption (independent censoring) deserves special scrutiny''. Therefore, it is necessary to relax the independence assumption and accommodate for dependent censoring.

\section{Derivation of Survival Likelihood}
Without assumption on the censoring mechanism, we have 
\begin{align}
    \mathcal{L}_{\text{dep}}=& \Pr(T=t, U > t |X)^{\delta} \cdot  \Pr(T > t, U=t |X)^{1-\delta} \nonumber\\
    = & \left\{ -\frac{\partial}{\partial y} \Pr(T>y, U>t )\vert_{y=t} \right\}^{\delta} 
    \cdot  \left\{ -\frac{\partial}{\partial z} \Pr(T>t, U>z ) \vert_{z=t} \right\}^{1-\delta}
    \nonumber\\ 
    = & \left\{ \int_{t}^{\infty} f_{U|T,X}(u^\ast|t,X) f_{T|X}(t|X) du^\ast \right\}^\delta 
    \cdot \left\{ \int_{t}^{\infty} f_{T|U,X}(t^\ast|u,X) f_{U|X}(u) dt^\ast \right\}^{1-\delta} \nonumber\\
    = & \left\{ f_{T|X}(t|X) \int_{t}^{\infty} f_{U|T,X}(u^\ast|t)  du^\ast \right\}^\delta 
    \cdot \left\{ f_{U|X}(t|X) \int_{t}^{\infty} f_{T|U,X}(t^\ast|u)  dt^\ast \right\}^{1-\delta} \nonumber\\
    = & \left\{ f_{T|X}(t|X) \left. \frac{\partial}{\partial u_1} \mathcal{C}(u_1,u_2) \right\vert_{\substack{ u_1=S_{T|X}(t|X) \\ u_2 = S_{U|X}(t|X) }} \right\}^{\delta} \\
    &
    \cdot \left\{ f_{U|X}(t|X) \left. \frac{\partial}{\partial u_2} \mathcal{C}(u_1,u_2) \right\vert_{\substack{ u_1=S_{T|X}(t|X) \\ u_2 = S_{U|X}(t|X) }} \right\}^{1-\delta}
\label{equation:likelihood_dep_appendix}
\end{align}

When conditional independent censoring is assumed, i.e., $T\indep C \vert X$, the likelihood can be written as
\begin{align}
    \mathcal{L}_{\text{indep}} 
    =& \Pr(T=t, U > t |X)^{\delta} \cdot  \Pr(T > t, U=t |X)^{1-\delta} \nonumber\\
    =& \{\Pr(T=t|X)\Pr(U>t|X)\}^\delta \cdot \{\Pr(T>t|X)\Pr(U=t|X)\}^{1-\delta} \nonumber\\
    =& \{f_{T|X}(t|X)S_{U|X}(t|X)\}^\delta \cdot \{f_{U|X}(t|X)S_{T|X}(t|X)\}^{1-\delta} \label{equation-indep_full}\\
    \propto & f_T(t|X)^\delta S_T(t|X)^{1-\delta}.
\end{align}

Contrasting Equation \ref{equation:likelihood_dep_appendix} and \ref{equation-indep_full}, the potential bias caused by the independence censoring assumption is evident. The independent likelihood is only unbiased if the copula partial derivatives equal to the marginal survival functions, i.e., $\mathcal{C}(u_1,u_2) = u_1 u_2$. Since the dependency structure is usually unknown in observational studies and unverifible from the data, the independent censoring assumption should be avoided in most practical applications.  

\section{Copulas}
Copulas are distribution functions of $d$-dimensional multivariate distributions with uniform margins. More formally, a function $C(u_1,\cdots,u_d):[0,1]^d \to [0,1]$ is a copula if the following hold \cite{Nelsen1998}:
\begin{itemize}
   \item (Grounded) The copula equals to $0$ if any of its argument is $0$, i.e., \[\mathcal{C}(u_1,\cdots,u_{i-1},0,u_{i+1},\cdots,u_d)=0.\]
   \item (Uniform) The copula equals to $u$ if one argument is $u$ and all others are $1$, i.e., \[\mathcal{C}(1,\cdots,1,u,1,\cdots,1)=u.\]
   \item ($d$-increasing) For all $u=(u_1,\cdots,u_d)$ and $v=(v_1,\cdots,v_d)$ where $u_i<v_i$ for all $i\in [d]$, $$\sum\limits_{(w_1,\cdots,w_d)\in\times_{i=1}^d \{u_i,v_i\} } (-1)^{\vert i:w_i=u_i\vert} \mathcal{C}(w_1,\cdots,w_d) \geq 0 $$
   \item ($d$-increasing) For each hyperrectangle $B=\Pi_{i=1}^d [u_i,v_i] \subseteq [0,1]^d$, the C-volume of $B$ is non-negative, i.e., $\int_B dC(u) \geq 0$. 
\end{itemize}
Intuitively, the $d$-increasing property states that the probability of any $d$-dimensional hyperrectangle is non-negative.

Copulas have been widely applied to areas such as quantitative finance and engineering thanks to \emph{Sklar's theorem} \cite{Sklar1959}, which states that any $d$-dimensional continuous joint distribution can be \textit{uniquely} expressed with $d$ uniform marginals and a copula $\mathcal{C}$. Formally, Sklar's theorem states that

Sklar's theorem \cite{Sklar1959} states that there exist a unique copula $\mathcal{C}$ such that the joint distribution of event and censoring can be expressed as
\begin{equation}
	F_{T,C}(t,c) = \mathcal{C} (F_T(t), F_C(c)).
	\label{equation-sklar}
\end{equation}
According to probability integral transform, we know that $F_T(t)$ and $F_C(c)$ are uniform in $[0,1]$. 
Note that instead of Equation \ref{equation-sklar} we can also assume a survival copula such that $P(T>t, C>c) = \tilde{\mathcal{C}}(1-F_T(t), 1-F_C(c))$ where $\tilde{C}(u,v)= u+v-1+C(1-u,1-v)$. However, as there is no substantial difference between these two expressions \cite{Nelsen1998}.

\section{Theorem Proofs}
For the sake of conciseness, we omit the covariate $X$ from the notations. 
As we do not assume conditionally independent censoring, the derivations with covariates follow similarly. 
\setcounter{theorem}{4}

\begin{lemma}
Suppose $\varphi$ is differentiable on $(0,1)$. If $\lim_{u\to 0} \varphi' (u) \in (-\infty, 0)$, then 
\[\lim_{t\to 0} \left. \frac{\partial}{\partial u_1} \mathcal{C}(u_1,u_2) \right\vert_{\substack{ u_1=S_{T}(t) \\ u_2 = S_{U}(t) }} = 1.
\]
\end{lemma}
\begin{proof}
From the definition of $S(t)$ we know  that $\lim_{t\to 0} S_T(T) = \lim_{t\to 0} S_U(T) =1$.
Furthermore, we have $\lim_{t\to 1}\varphi^{-1}(t)=0$ and $\varphi(0)=1$ from the definition of the generator. Lastly, utilizing the definition of Archimedean copula, i.e., $\mathcal{C}=\varphi(\varphi^{-1}(u_1)+ \varphi^{-1}(u_2))$, we can write
\begin{align}
     & \lim_{t\to 0} \left. \frac{\partial}{\partial u_1} \mathcal{C}(u_1,u_2) \right\vert_{\substack{ u_1=S_{T}(t) \\ u_2 = S_{U}(t) }} \nonumber\\ 
    =& \lim_{t\to 0} \frac{\varphi'( \varphi^{-1}(S_{U}(t))+ \varphi^{-1}(S_{T}(t))) }{ \varphi'( \varphi^{-1}(S_{T}(t))} =  \lim\limits_{t\to 0} \frac{\varphi'(t)}{\varphi'(t)}= \frac{c}{c} = 1
\end{align}
The case for $\partial \mathcal{C}(u_1, u_2)/\partial u_2$ can be similarly shown.
\end{proof}

\begin{theorem}
All copulas expressed by the neural network specified by $\varphi^{nn}$ satisfy condition (C2) of Theorem \ref{theorem-identifiability}.
\end{theorem}
\begin{proof}
The theorem can be proved by induction with respect to the indexes of the hidden layers. 
For notation simplicity, we assume that each hidden layer contains the same number of $H$ hidden units. The results can be straightforwardly extended to the general cases.
If $\varphi^{nn}$ has one hidden layer, i.e. $L=1$, the output from the last hidden layer of $\varphi^{nn}$ can be straightforwardly expressed as
\begin{align}
    \varphi^{nn}_{1}(u) = \sum\limits_{k=1}^{H} A_{1,j,k} e^{-B_{1,j} \cdot u} .
\end{align}
Since the output layer performs convex combination of $\varphi_1^{nn}(u)$, $B_{i,j}$ is positive, $A_{l,j,k}$ is non-negative with $\sum_j A_{l,i,j}=1$, we can express $\varphi^{nn}(u)$ as $\varphi^{nn}(u) = \sum\limits a \cdot \exp (-b\cdot t)$ with some $a, b>0$.
It follows naturally that $\varphi'(u) <0$ for all $u$. Utilizing Lemma 4, the theorem is satisfied when $L=1$.

Assume that $\varphi^{nn}(u)$ has $L$ hidden layer, and the output of its last hidden layer can be expressed as $\varphi^{nn}_L(u) = \sum\limits a \cdot \exp (-b\cdot t)$ with some $a, b>0$, we have $\varphi'(u)<0$ when $L=n+1$ using the fact that the output layer is a convex combination of $\varphi^{nn}_L(u)$.
Utilizing Lemma 4 again, the theorem is satisfied when for $L=n+1$, completing the induction. 
\end{proof}

\newcommand{\algrule}[1][.2pt]{\par\vskip.5\baselineskip\hrule height #1\par\vskip.5\baselineskip}
\begin{algorithm}[t]
    \caption{Survival Modeling with Unknown Dependent Censoring}
    \textbf{Input:} $\mathcal{D}$: right-censored survival dataset $\mathcal{D}= (X_i, t_i, \delta_i)$ for $i=1,\dots,N$; 
                    $\mathcal{M}_T$: parametric survival model that emits the survival marginals $\hat{S}_{T|X}(t|X), \hat{f}_{T|X}(t|X)$; 
                    $\mathcal{M}_U$: parametric survival model that emits the censoring marginals $\hat{S}_{T|X}(t|X), \hat{f}_{U|X}(t|X)$;
                    Learning rate $\alpha$; Number of iterations $S$
    \\
    \textbf{Output:} $\varphi_{nn}, \psi_T, \psi_U$, the learned copula, survival, and censoring models
    \algrule
    \begin{algorithmic}[1]
        \State $\mathcal{M}_T \leftarrow$ Instantiate the survival marginal $(M_1; \psi_T)$;
        \State $\mathcal{M}_U \leftarrow$ Instantiate the censoring marginal $(M_2; \psi_U)$;
        \For{$i=1,\dots,S$}
        \State Compute $\mathcal{L}_{\text{dep}}$ for a mini-batch of $\mathcal{D}$;
        \State $\psi_T,\psi_U \leftarrow \text{AdamWUpdate}(\mathcal{L}_{\text{dep}}, (\psi_T, \psi_U), \alpha)$;
        \State $\Phi \leftarrow \text{AdamWUpdate}(\mathcal{L}_{\text{dep}}, \varphi_{nn}, \alpha)$;
        \EndFor
    \end{algorithmic}
\label{algo:algorithm}
\end{algorithm}

\section{Additional Experiment Details and Results}

\subsection{Experiment Setting and Implementation Details}
Experiments are conducted on a PC with one NVIDIA RTX4090 GPU.
We utilize the Pytorch for implementing all neural networks and automatic differentiation. 
Tensors are computed with double precision (fp64) as the inversion of $\varphi$ mandates numerical precision. 
When using Newton’s method to compute the inverse $\varphi_{nn}^{-1}$, we terminate when the error is less than $\num{1e-12}$. 

For all our experiments we set $\varphi_{nn}$ with $L = 2$ and $H_1 = H_2 = 10$, i.e., the copula representation contains two hidden layers with each of width $10$. 
The network is small but sufficient for the dependency structure since the generator $\varphi_{nn}$ is only $1$-dimensional. 
$\Phi_B$ and $\Phi_B$ were uniformly initialized in the range $[0, 1]$ and $(0, 2)$. 
We use AdamW \cite{Loshchilov2018} for optimization and use 30\% of the training samples for evaluating the validation log-likelihood. 
No further hyperparameter tuning was performed. An illustration for the network of $\varphi_{nn}$ can be found in Figure \ref{figure:illustration-network}.

\begin{figure*}
    \centering
    \includegraphics[width=\linewidth]{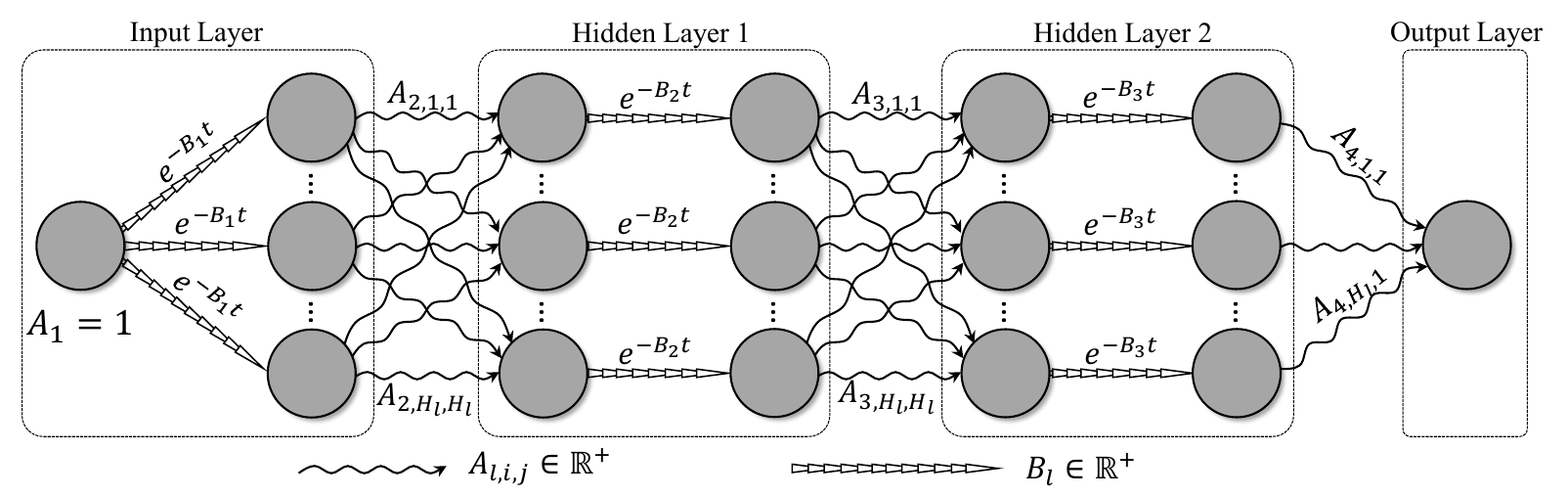}
    \caption{Network structure for representing the Archimedean copula generator $\varphi_{nn}$. The coefficients of $A$ are parameterized by Softmax to ensure a convex combination. $B$ is non-negative such that $\exp(-Bt)$ are negative exponentials.}
    \label{figure:illustration-network}
\end{figure*}

\begin{table*}[t]
\centering
\caption{The Archimedean copulas used for generating dependence among survival and censoring outcomes and evaluating the performances of {\ours}. For Frank copula, $D_1$ is the Debye function of the first kind, i.e., $D_1(\theta) = \frac{1}{\theta}\int^\theta \frac{t}{e^t-1}dt$.}
\begin{tabular}{l l  l l  l}
\hline
Copula & $C_\theta (u,v)$ & $\varphi_\theta(t)$ & Kendall's $\tau$ \\
\hline
Clayton & $(u^{-\theta} + v^{-\theta}  -1 )^{\frac{1}{\theta}}$ & $\frac{1}{\theta}(t^{-\theta} - 1)$ & $\frac{\theta}{\theta+2}$ \\
Gumbel & $\exp[-((\log(u))^{\theta} + (-\log(u))^{\theta})^{\frac{1}{\theta}}]$ & $(-\log t)^\theta $ & $\frac{\theta-1}{\theta}$ \\
Frank & $-\frac{1}{\theta} \log[1+\frac{(e^{-\theta u} -1)(e^{\theta v)-1}}{e^{-\theta }-1}$ & $-\log\frac{e^{-\theta t}-1}{e^{-\theta}-1}$& 1+ $\frac{4}{\theta}(D_1(\theta)-1)$ \\
Independent  & $uv$ & $-\log t$ & $0$ \\
\hline
\end{tabular}
\label{table:copulas}
\end{table*}

\begin{algorithm}[hbt]
    \caption{Synthetic Data Generating Process}\label{euclid}
    \textbf{Input:} $X$: set of covariates; $\nu_T$, $\rho_T$; $\nu_U$, $\rho_U$: parameters of the marginal distribution; $\psi_T(X)$, $\psi_U(X)$: risk functions; $C_\theta$: a pre-defined copula parameterized by $\theta$.
    \\
    \textbf{Output:} $\mathcal{D}= (X_i, t_i, \delta_i)$ for $i=1,\dots,N$.
    \algrule
    \begin{algorithmic}[1]
        \State $\mathcal{D}=\emptyset$;
        \For{$i=1,\dots,N$}
        \State Sample $u_{1}\sim \mathcal{U}_{[0,1]}$ and $u_{2}\sim \mathcal{U}_{[0,1]}$ such that $F(u_{1},u_{2})=\mathcal{C}_\theta( F(u_{1}, F(u_{2}))$;
        \State $T\leftarrow \left(\frac{-\log u_1}{\exp \psi_T(X)}\right)^{\frac{1}{\nu_T}}\cdot \rho_T$;
        \State $U\leftarrow \left(\frac{-\log u_2}{\exp \psi_U(X)}\right)^{\frac{1}{\nu_U}}\cdot \rho_U$;
        \State $D \leftarrow D \cup \{(X_i, \min(T_i,U_i), \mathbb{I}_{[T_i<U_i]})\}$.
        \EndFor
    \end{algorithmic}
\label{algo:dgp}
\end{algorithm}

\begin{algorithm}[hbt]
    \caption{Semi-Synthetic Censoring Inducing Procedure}\label{euclid}
    \textbf{Input:} $X$: set of covariates; $T$: set of dependent variables; $\mathcal{C}:$ A copula specifying the dependency structure.
    \\
    \textbf{Output:} $\mathcal{D}= (X_i, t_i, \delta_i)$ for $i=1,\dots,N$. A censored dataset where the joint distribution of $T$ and $U$ are governed by the copula $\mathcal{C}$.
    \algrule
    \begin{algorithmic}[1]
        \State $M_T \leftarrow \text{Weibull} (X,Y,\bm{1}^N)$; Fit event model using Weibulll distribution
        \State $M_U \leftarrow M_T$
        \State $M_U.\nu \leftarrow M_T.\nu/0.8$; Reduce variance for the censoring distribution
        \State $\mathcal{D}=\emptyset$;
        \For{$i=1,\dots,N$}
        \State $u_i\leftarrow S_T(Y_i)$; obtain event quantile
        \State $v_i \sim \mathcal{C}(\cdot \vert u_i)$  Conditionally sample censoring quantile from $\mathcal{C}$
        \State $U_i \leftarrow S_U(v_i)$ Obtain censoring time with probability inverse transform
        \State $D \leftarrow D \cup \{(X_i, \min(T_i,U_i), \mathbb{I}_{[T_i<U_i]})\}$.
        \EndFor
    \end{algorithmic}
\label{algo:dgp-semi}
\end{algorithm}

We implement {\ours} with \texttt{PyTorch} 2.0.0. The pseudo code for the algorithm procedure can be found in Algorithm \ref{algo:algorithm}. 
{\ours} uses the same hyper-parameters across all experiments except the learning rate. 
We set the learning rate to $1e-4$ for the synthetic datasets and the semi-synthetic datasets \texttt{STEEL} and \texttt{Airfoil}. For the real-world dataset \texttt{GBSG2} and \texttt{SEER}, we use the learning rate of $1e-5$.
All experiments are conducted with an NVIDIA RTX4090 GPU where the average training speed is approximately 3 epoch/s. However, as {\ours} uses double precision tensors, access to professional-grade GPU such as a NVIDIA V100 (7.8 TFLOPS for V100 vs 1.3 TFLOPS for RTX4090) will significantly boost the training speed. 

For implementation of the compared algorithms, we use the \textsc{PyCox} package \cite{Kvamme2019} which implements \textsc{DeepSurv}, \textsc{DeepHit}, \textsc{CoxCC}, \textsc{CoxTime} and \textsc{PCHazard}. For \textsc{Random Survival Forest} and \textsc{CoxPH} methods, we utilize the implementation provided in \textsc{Scikit-Survival} \cite{sksurv}.
The code of \textsc{KnownCopula} can be found at https://github.com/rgklab/copula\_based\_deep\_survival.
For classical survival methods such as \textsc{Random Survival Forest} and \textsc{CoxPH}, we use the default parameters. For all methods that utilize deep neural networks, we use the same multi-layer perceptron (MLP) structure of $[32,32,32]$ and $ReLU$ activation function. The same network structure is also used for the covariate representation in {\ours}, albeit {\ours} uses the hyperbolic tangent activation function. A validation set is used for all deep learning based algorithms, which consists of 30\% of the training samples.

To ensure a fair comparison among all algorithms, we use the random seed for all the data generation procedure and train/val/test set splits. In other words, for each of the ten repeats, every algorithm use the same training, validation and test samples.

\subsection{Datasets}
We provide the pseudocode for generating the synthetic datasets in Algorithm \ref{algo:dgp}, and the pseudocode for inducing censoring in the semi-synthetic datasets in Algorithm \ref{algo:dgp-semi}. 
Furthermore, the details of different Archimedean copulas used during simulation and the relationships between $\theta$ and Kendall's $\tau$ are provided in Table \ref{table:copulas}. 

For the real-world \texttt{GBSG2} dataset, it can be accessed using the \textsc{lifelines} package \cite{DavidsonPilon2023}.
The \texttt{SEER} dataset can be accessed using the \textsc{SeerStat} software downloaded from https://seer.cancer.gov/seerstat/.

\subsection{Results from Learning Dependence under Non-linear Risk}
We present the scatter plot, cumulative distribution plot, and log-probability density plots for learning the dependency structure from right-censored observations with non-linear risks in Figure \ref{fig:nonlinear-copula}. It can be seen that {\ours} successfully learns the underlying copulas governing the dependency, as the samples from learned models closely resemble the ground truth ones. 

\begin{figure*}
	\begin{subfigure}{.45\textwidth}
		\centering
		\begin{tabular}{ccc}
			\includegraphics[width=.32\linewidth]{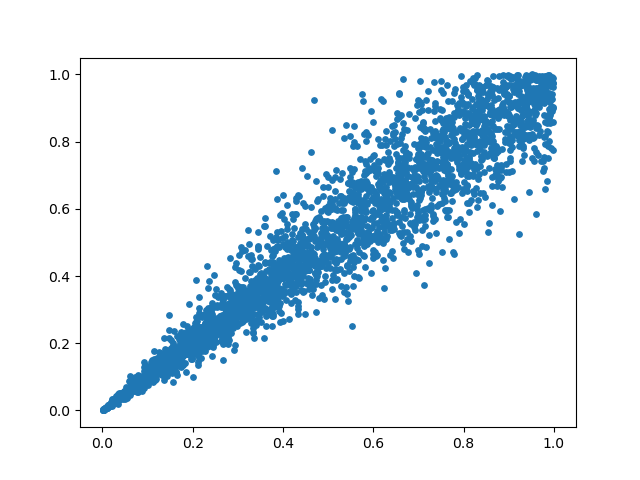} & 
			\includegraphics[width=.32\linewidth]{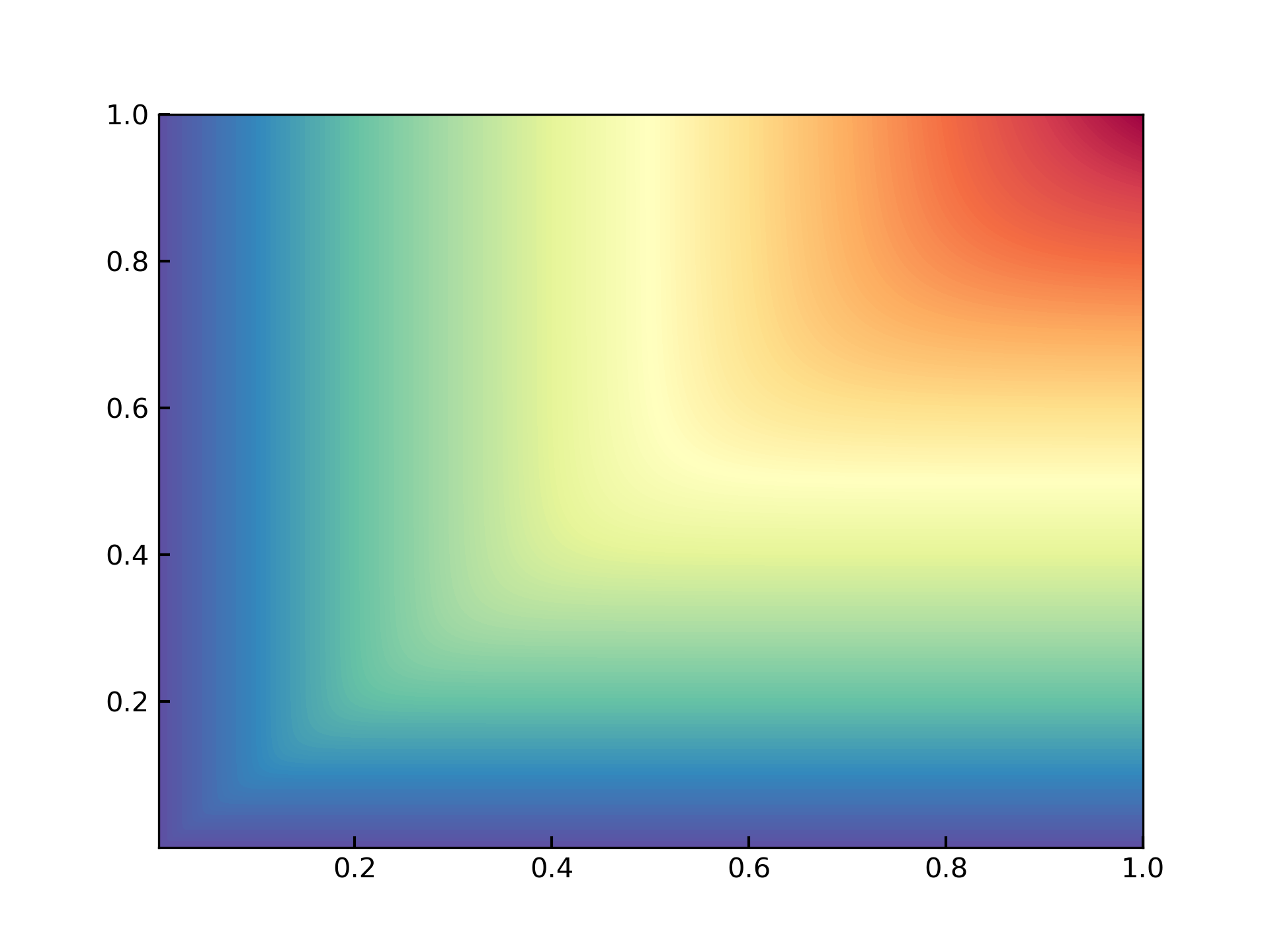} & 
			\includegraphics[width=.32\linewidth]{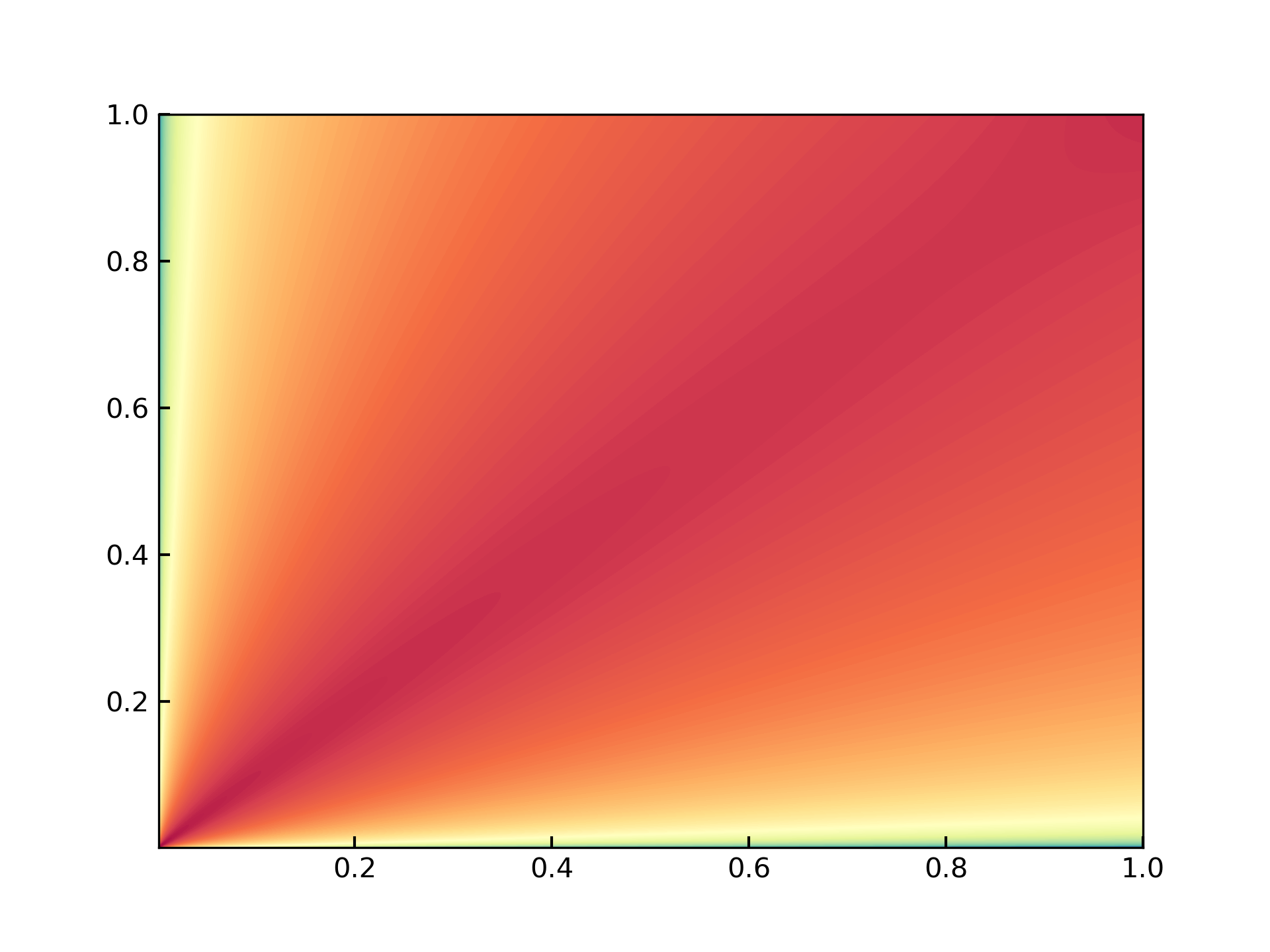}\\
		\end{tabular}
	\end{subfigure}
	\hfill
	\begin{subfigure}{.45\textwidth}
		\centering
		\begin{tabular}{ccc}
			\includegraphics[width=.32\linewidth]{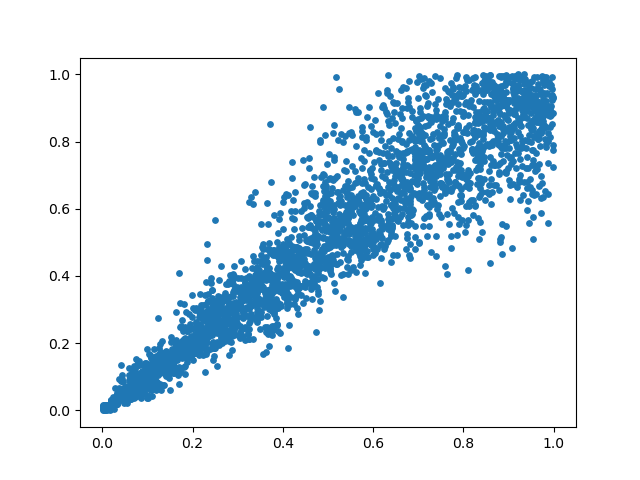} & 
			\includegraphics[width=.32\linewidth]{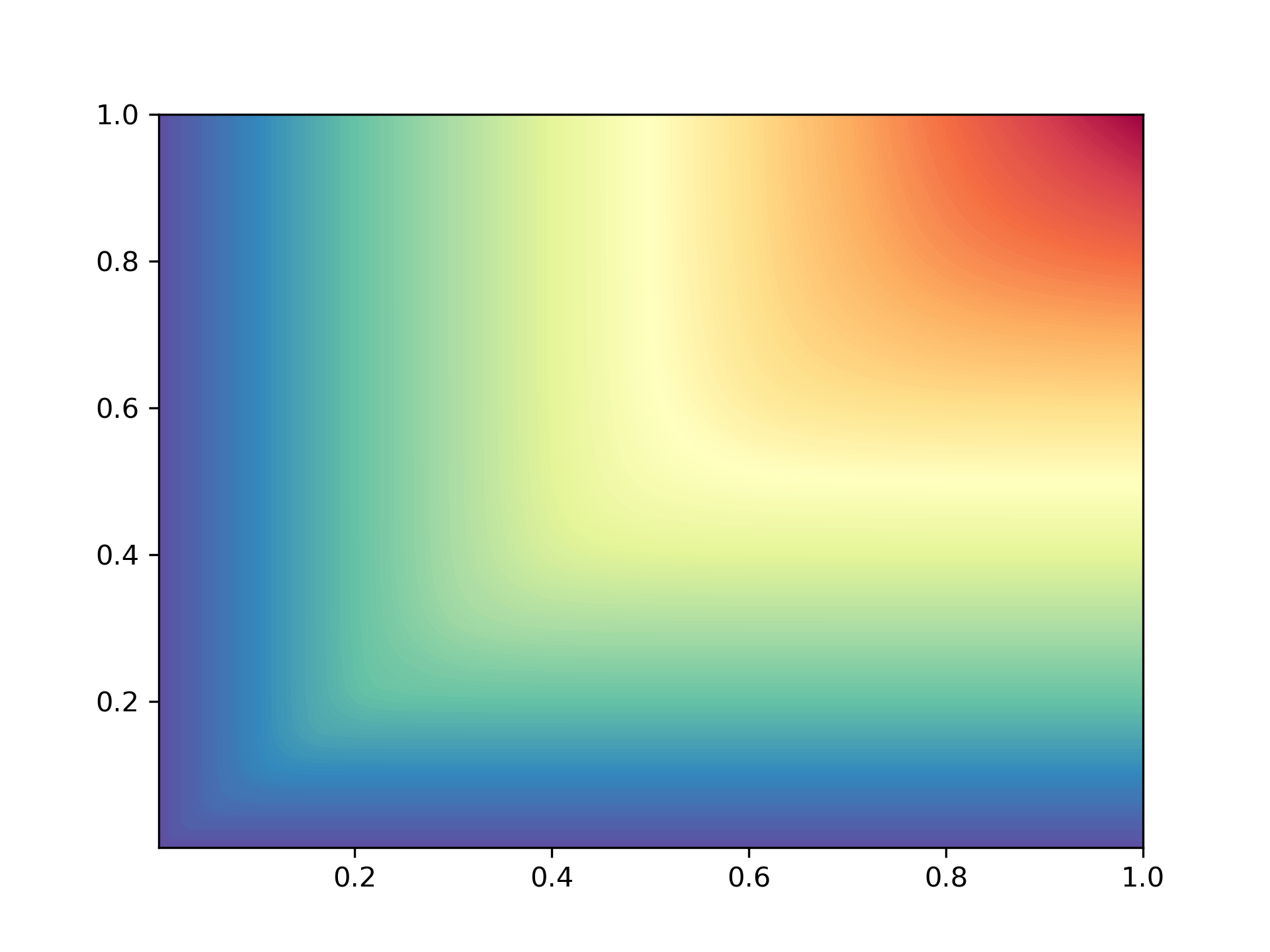} & 
			\includegraphics[width=.32\linewidth]{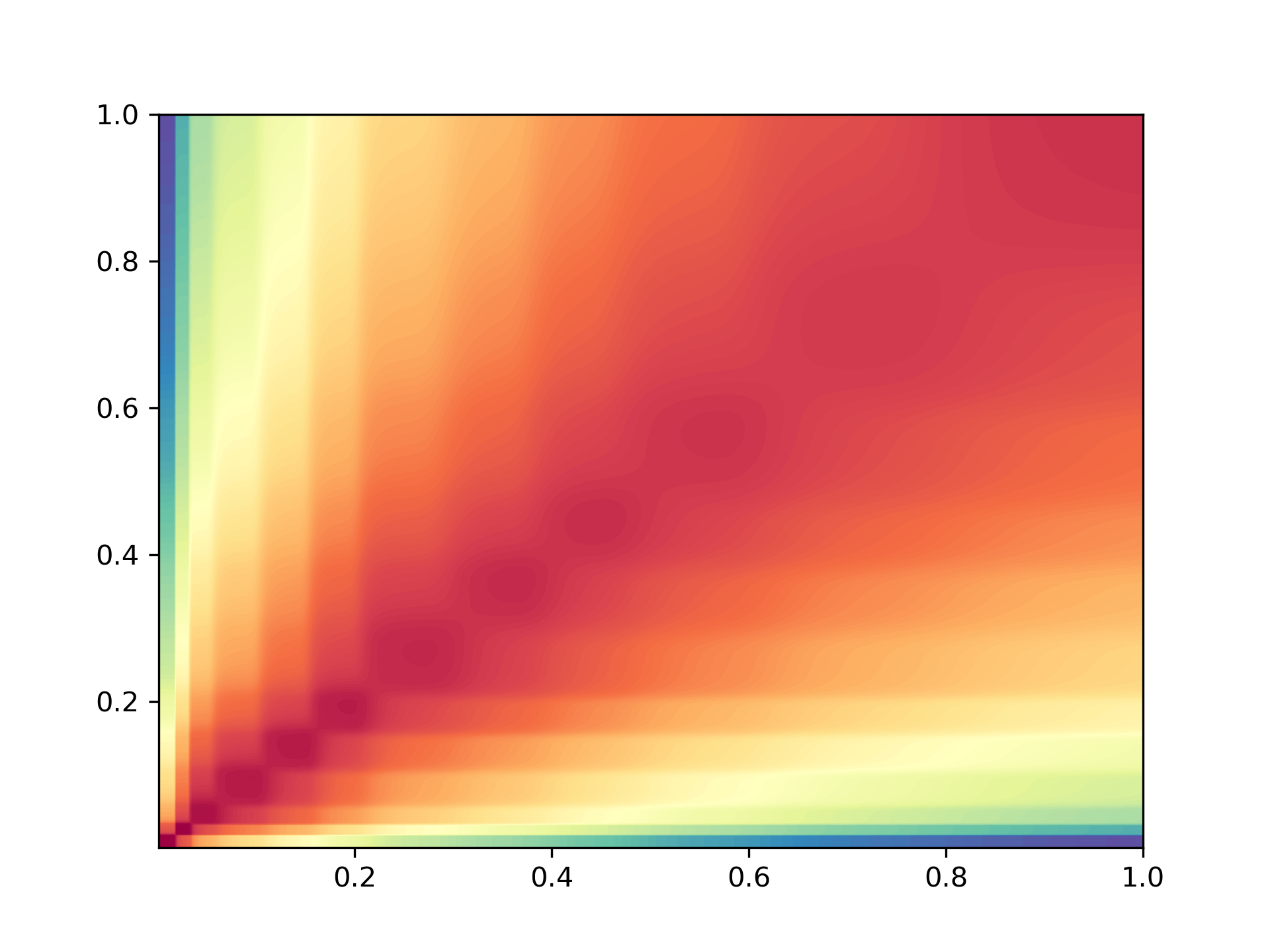} \\
		\end{tabular}
	\end{subfigure}
	\\
	\begin{subfigure}{.45\textwidth}
		\centering
		\begin{tabular}{ccc}
			\includegraphics[width=.32\linewidth]{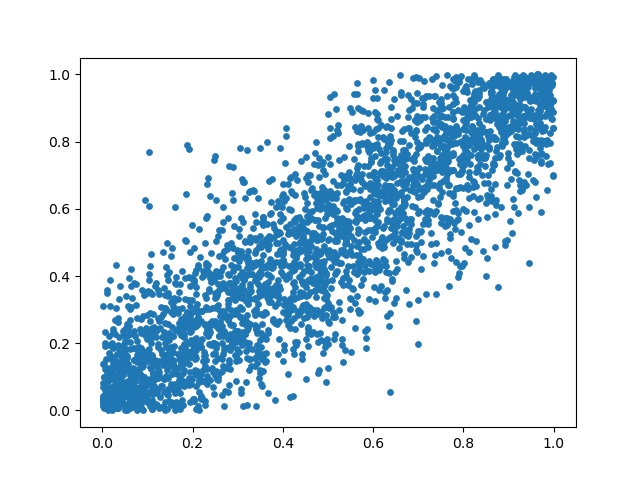} & 
			\includegraphics[width=.32\linewidth]{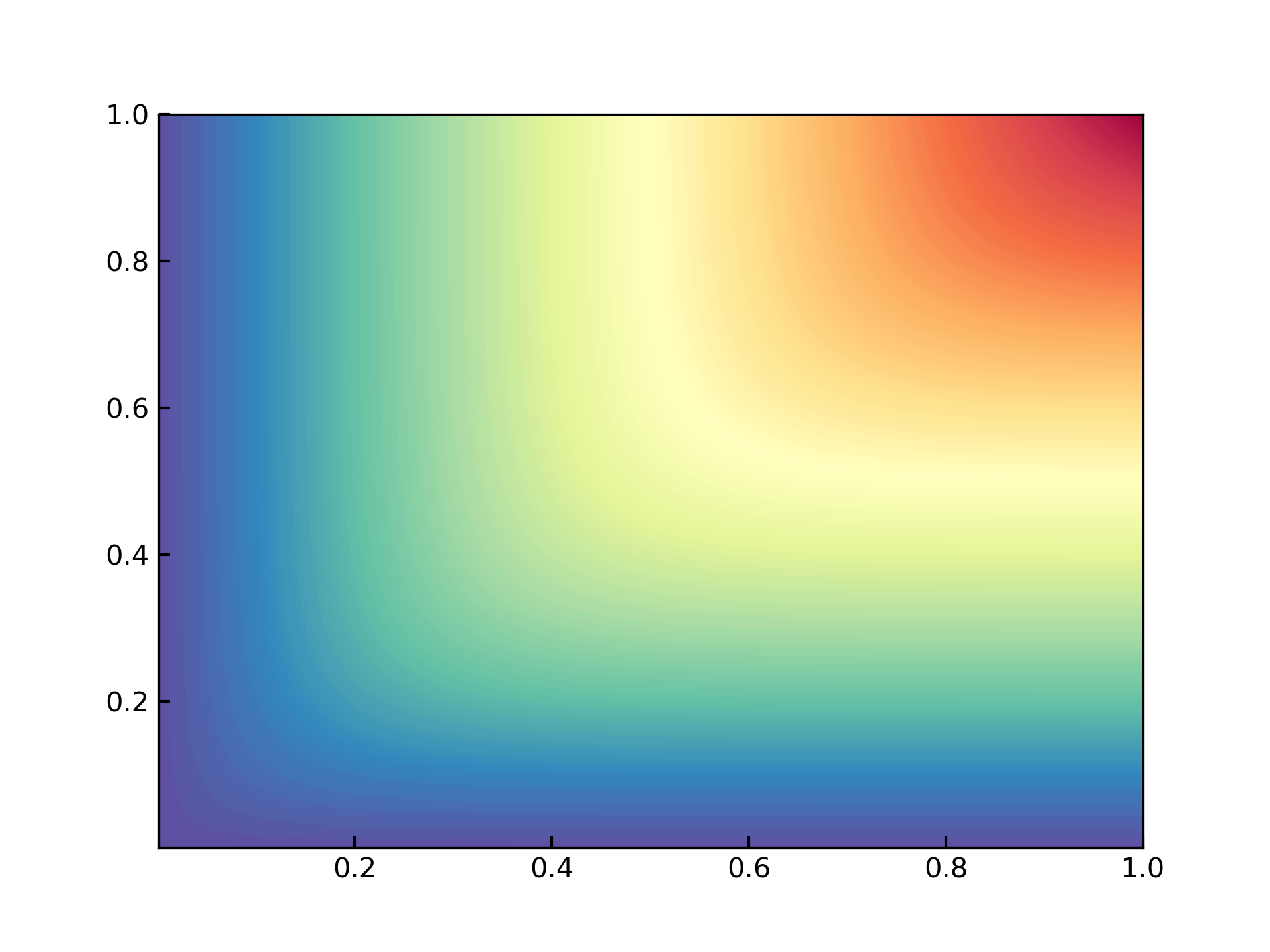} & 
			\includegraphics[width=.32\linewidth]{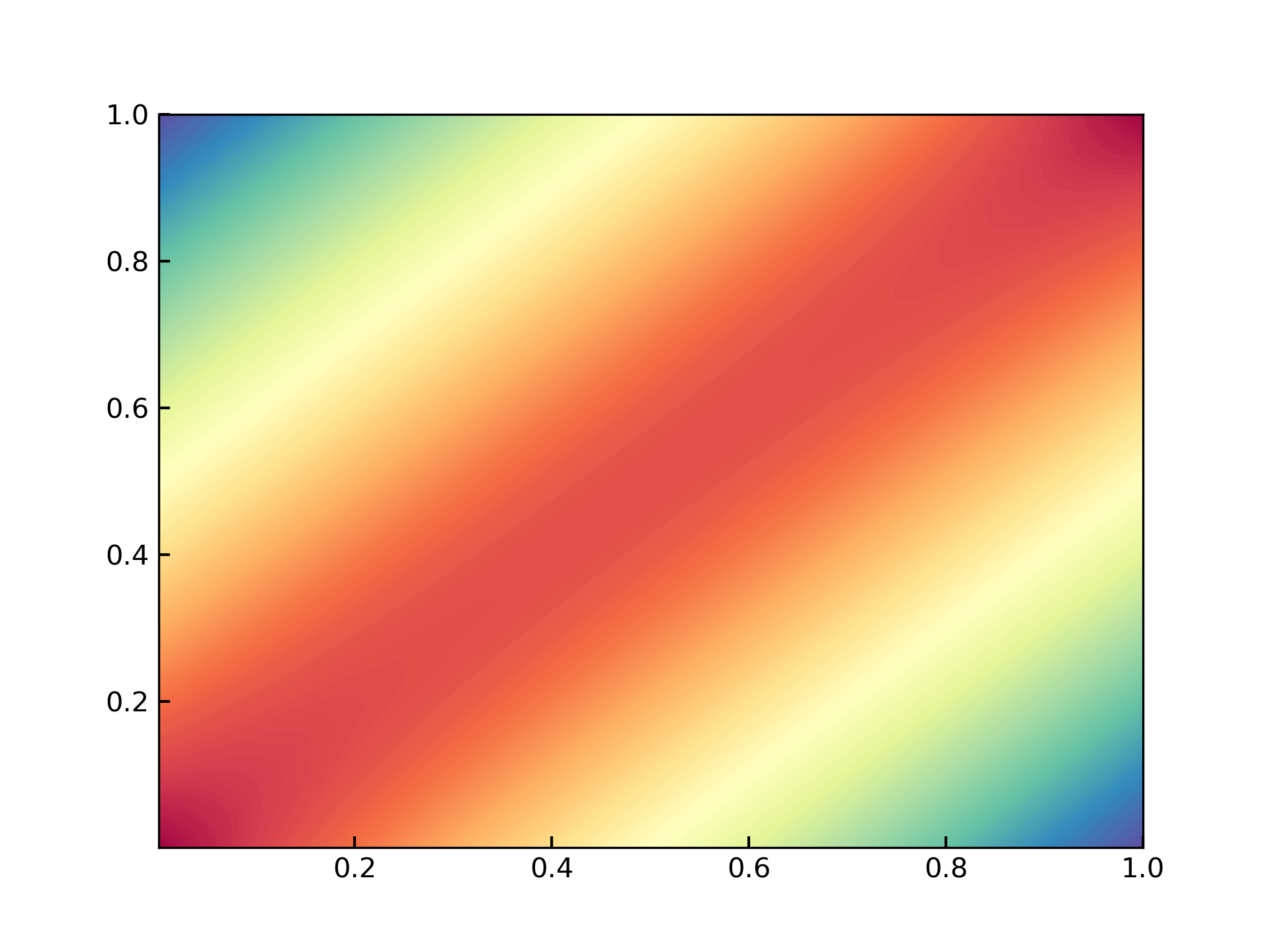}\\
		\end{tabular}
	\end{subfigure}
	\hfill
	\begin{subfigure}{.45\textwidth}
		\centering
		\begin{tabular}{ccc}
			\includegraphics[width=.32\linewidth]{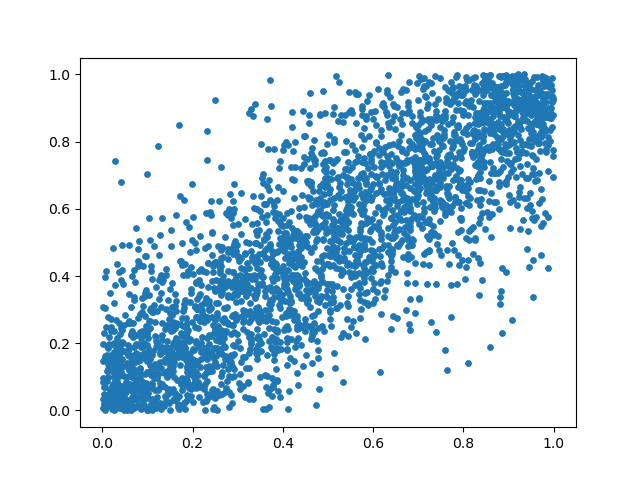} & 
			\includegraphics[width=.32\linewidth]{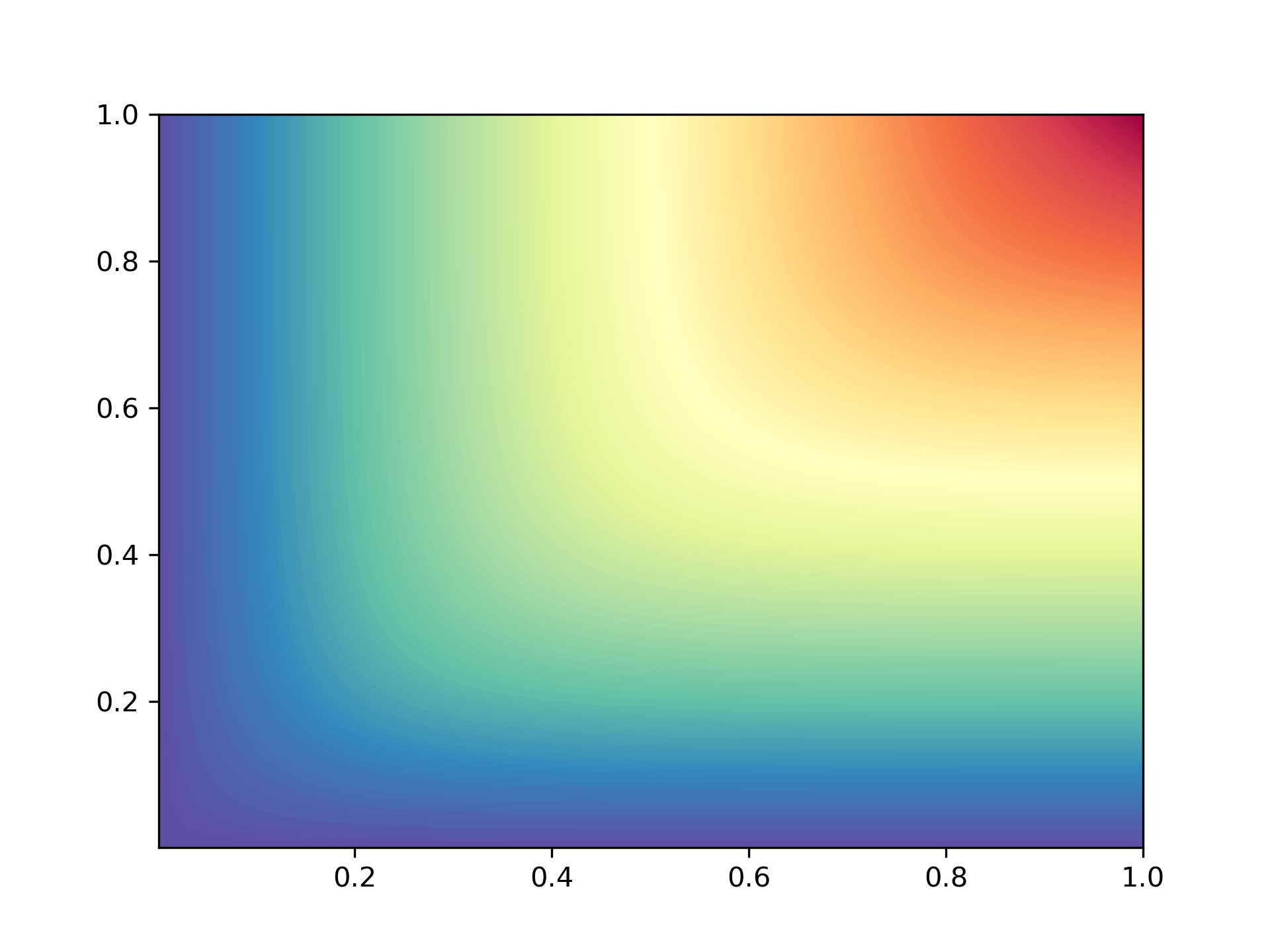} & 
			\includegraphics[width=.32\linewidth]{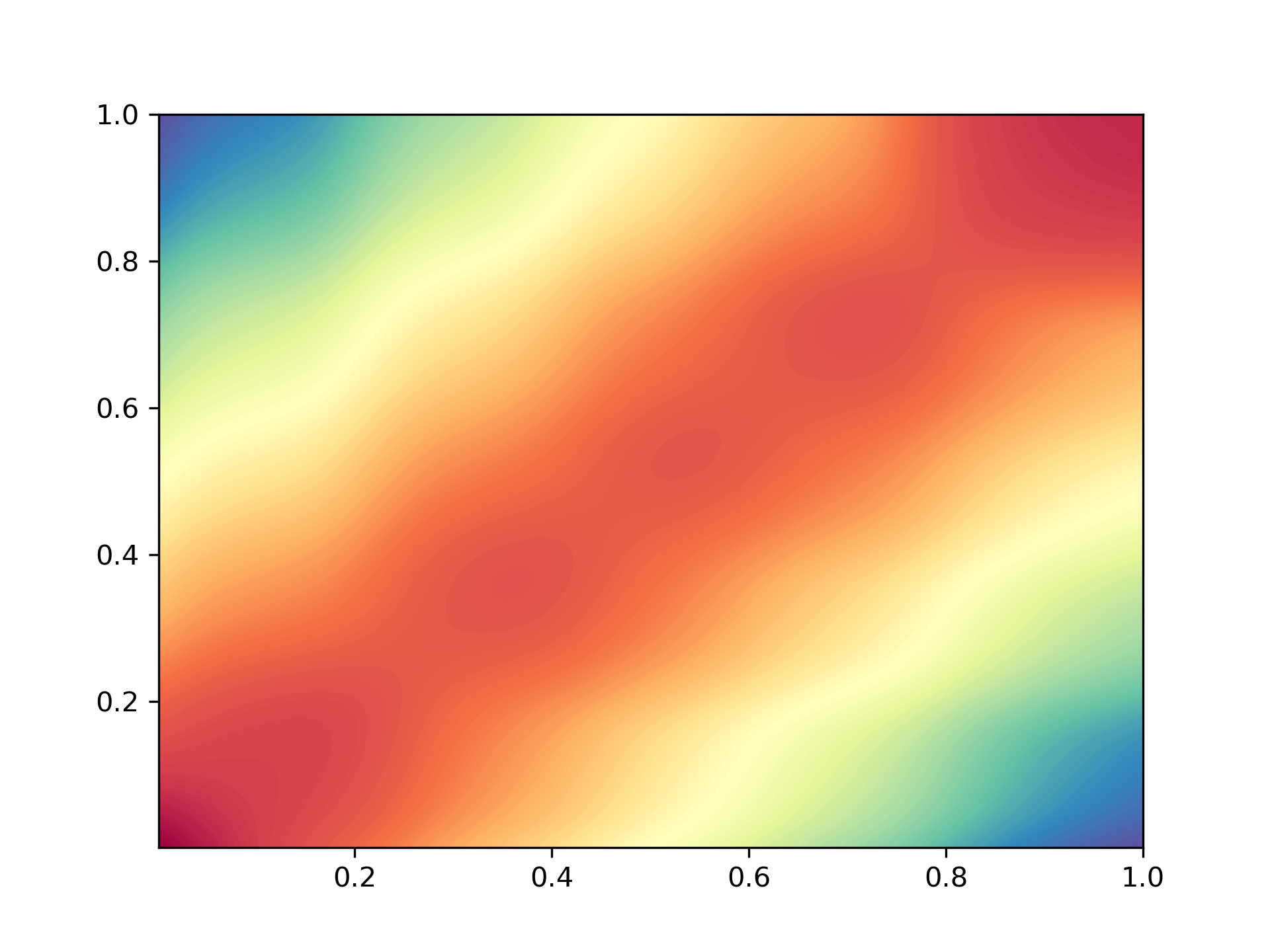} \\
		\end{tabular}
	\end{subfigure}	
	\\
	\begin{subfigure}{.45\textwidth}
		\centering
		\begin{tabular}{ccc}
			\includegraphics[width=.32\linewidth]{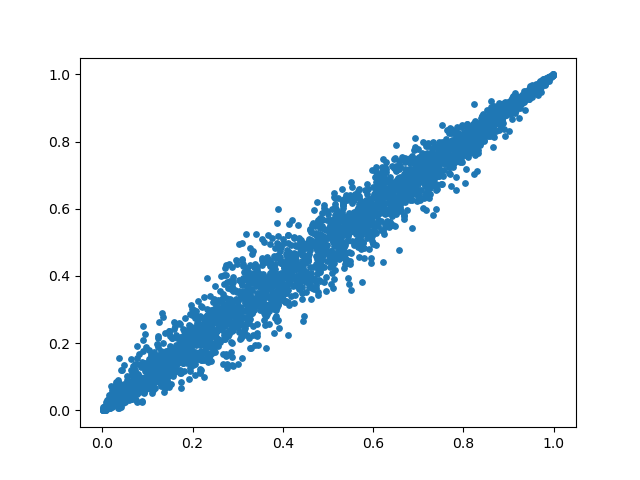} & 
			\includegraphics[width=.32\linewidth]{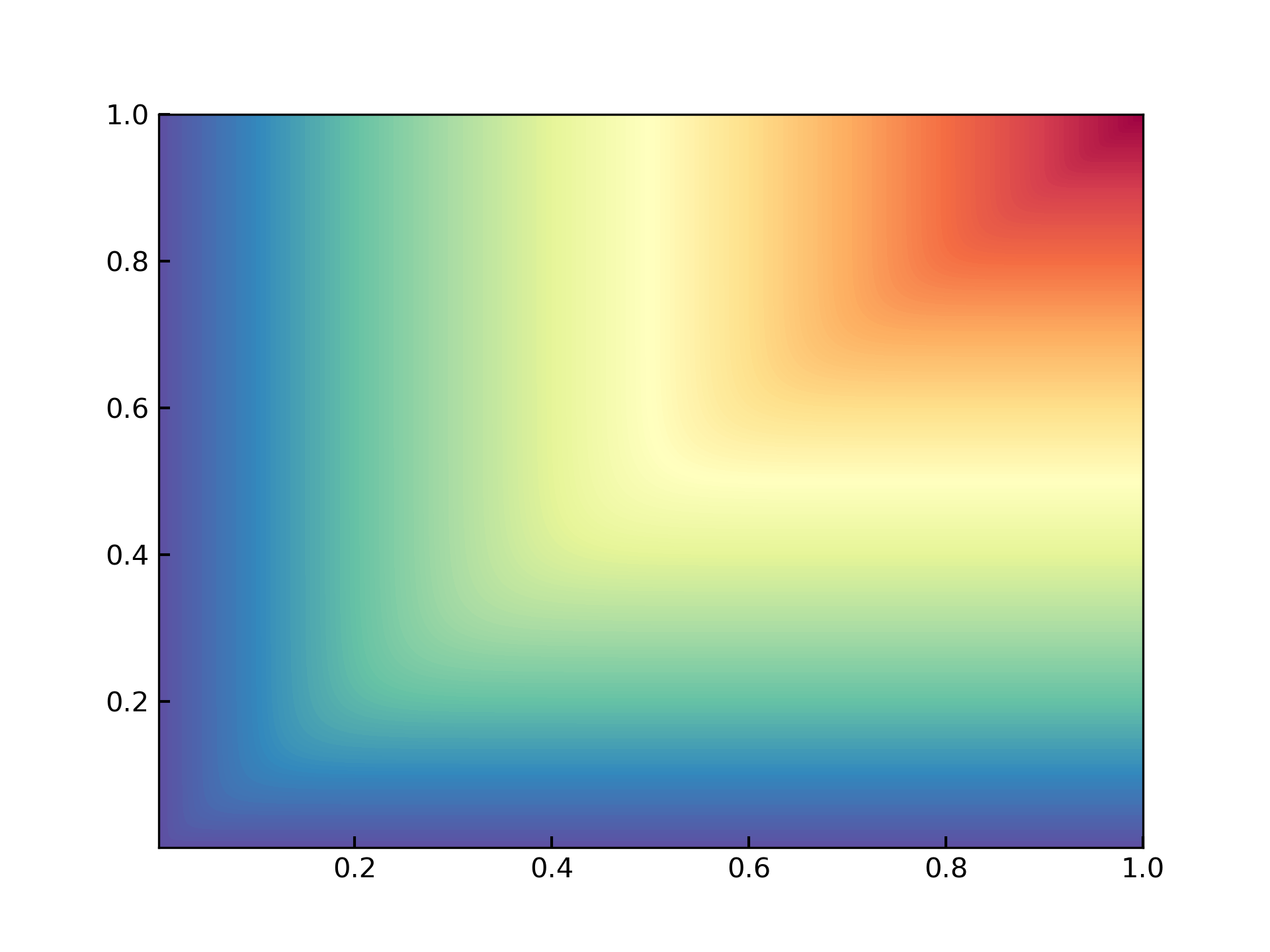} & 
			\includegraphics[width=.32\linewidth]{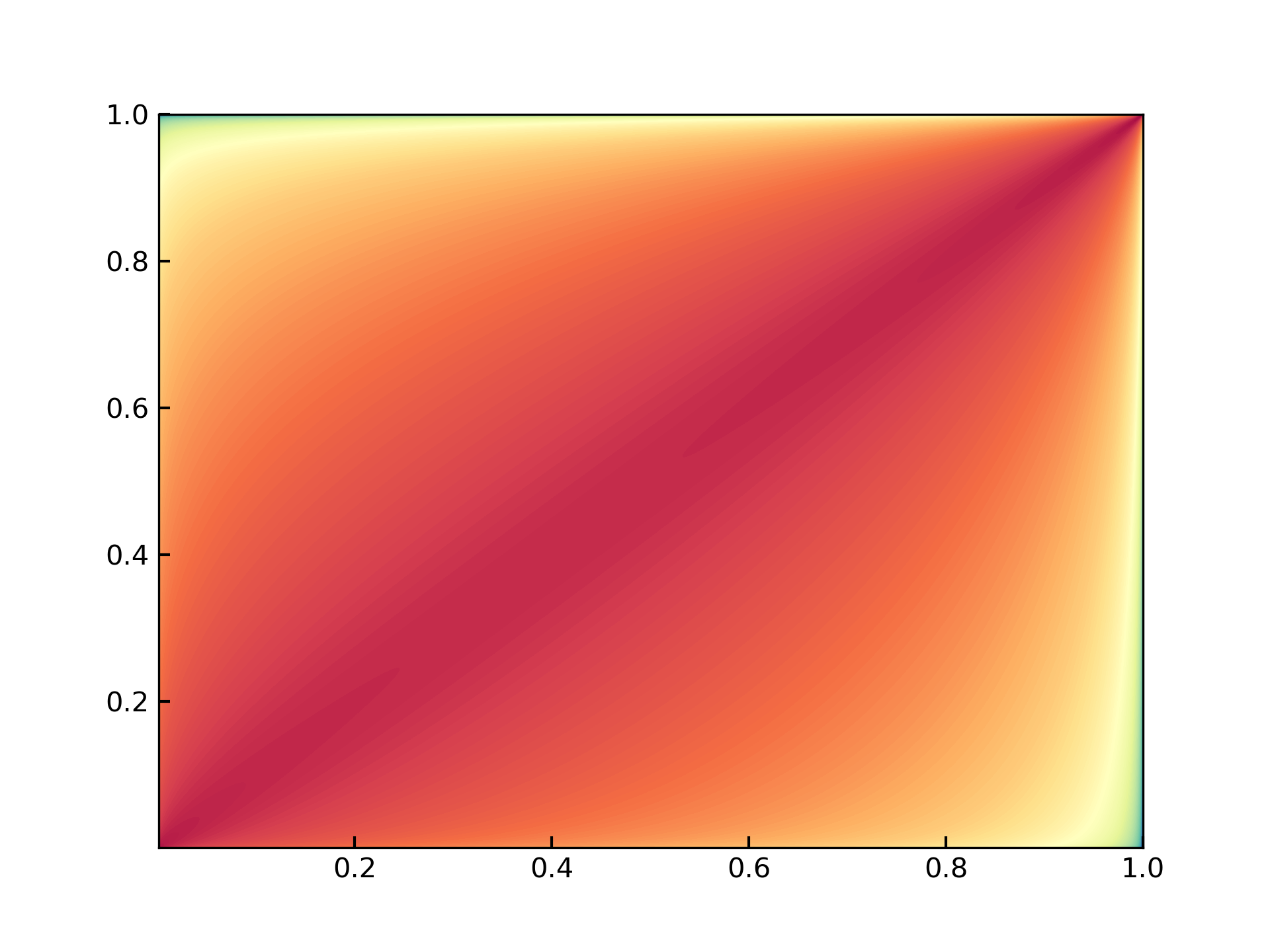}\\
		\end{tabular}
	\end{subfigure}
        \hfill
	\begin{subfigure}{.45\textwidth}
		\centering
		\begin{tabular}{ccc}
			\includegraphics[width=.32\linewidth]{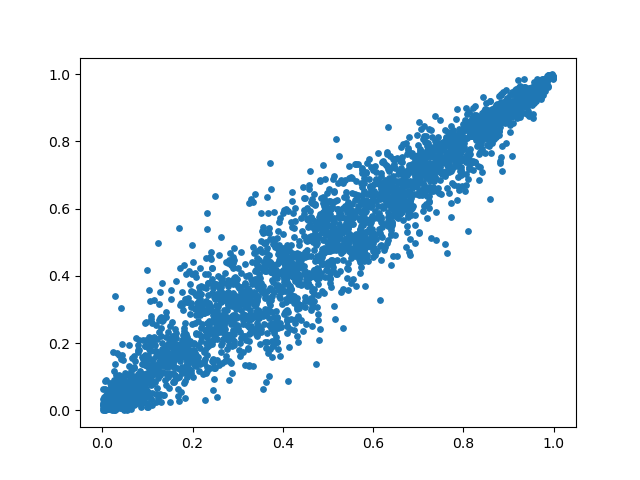} & 
			\includegraphics[width=.32\linewidth]{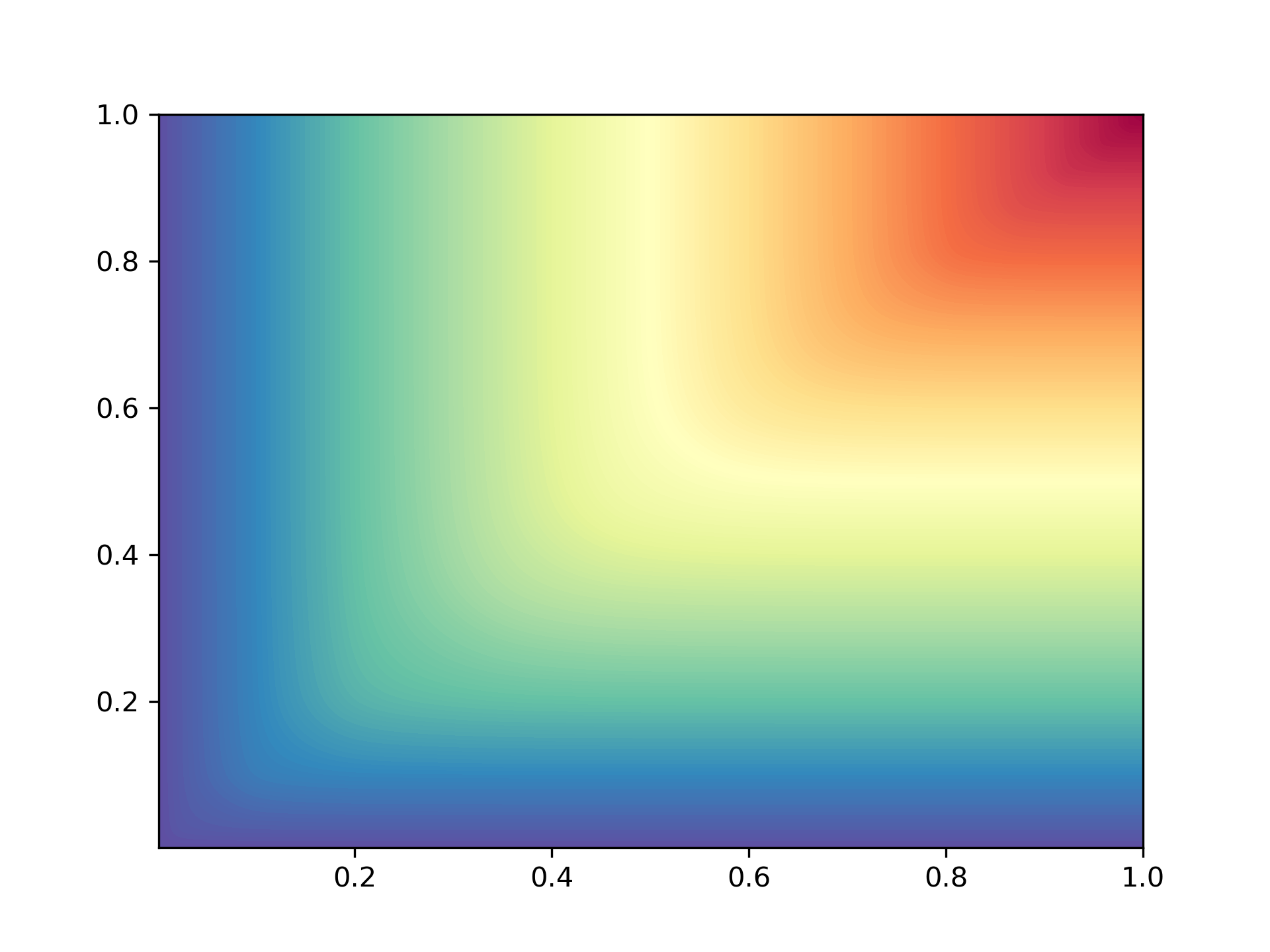} & 
			\includegraphics[width=.32\linewidth]{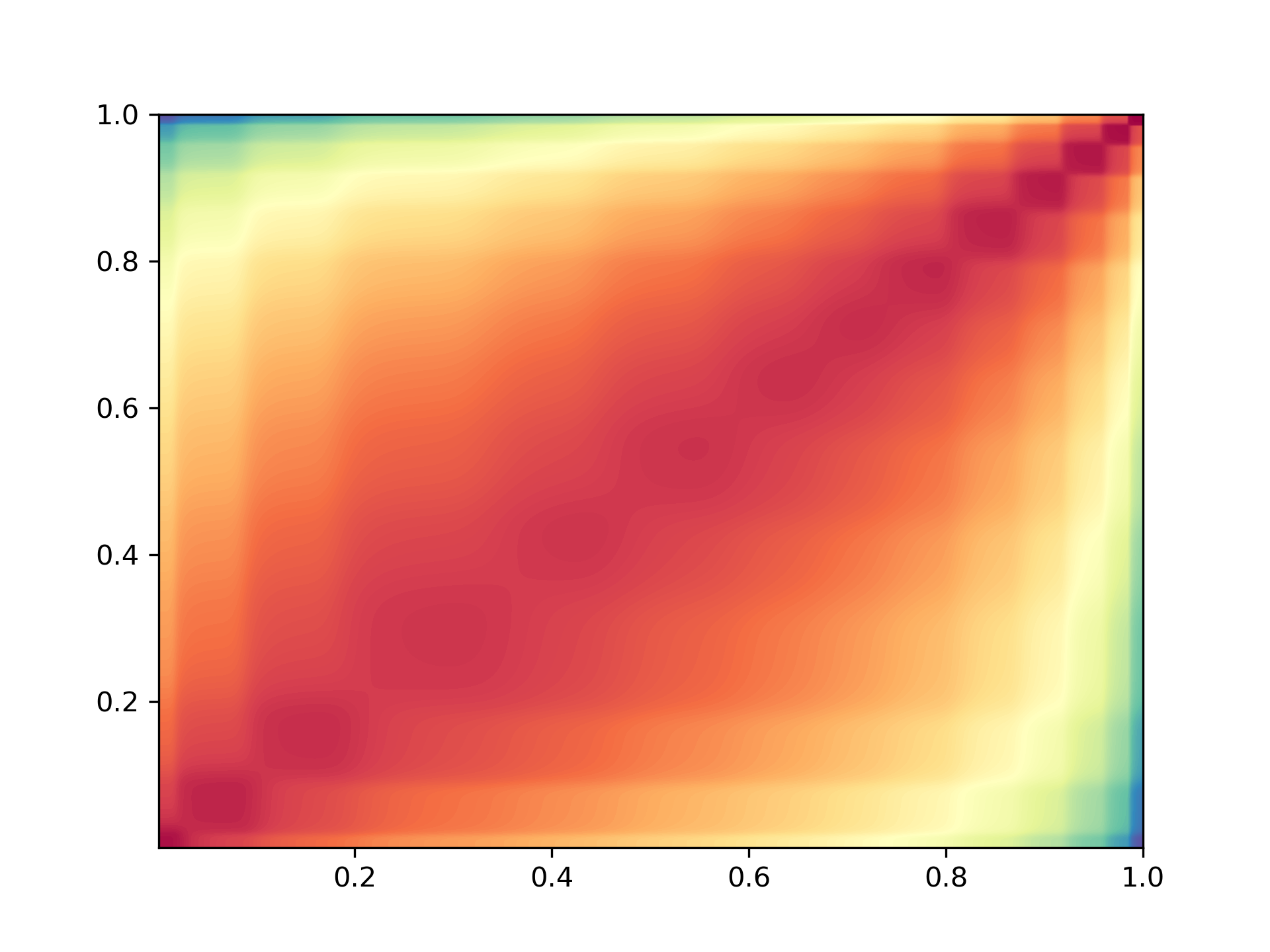} \\
		\end{tabular}
	\end{subfigure}
	\begin{subfigure}{.45\textwidth}
		\centering
		\begin{tabular}{ccc}
			\includegraphics[width=.32\linewidth]{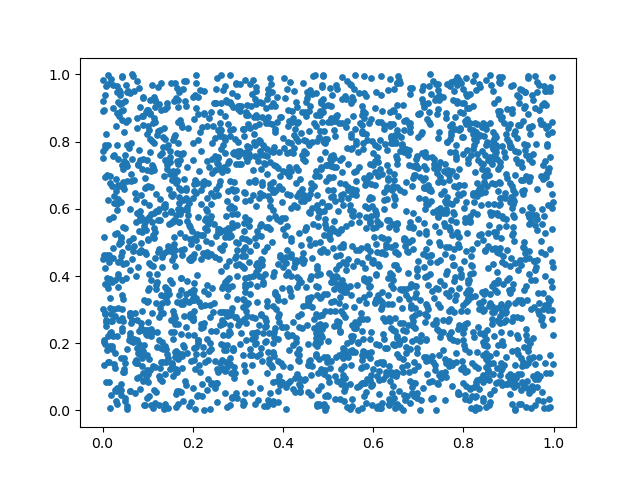} & 
			\includegraphics[width=.32\linewidth]{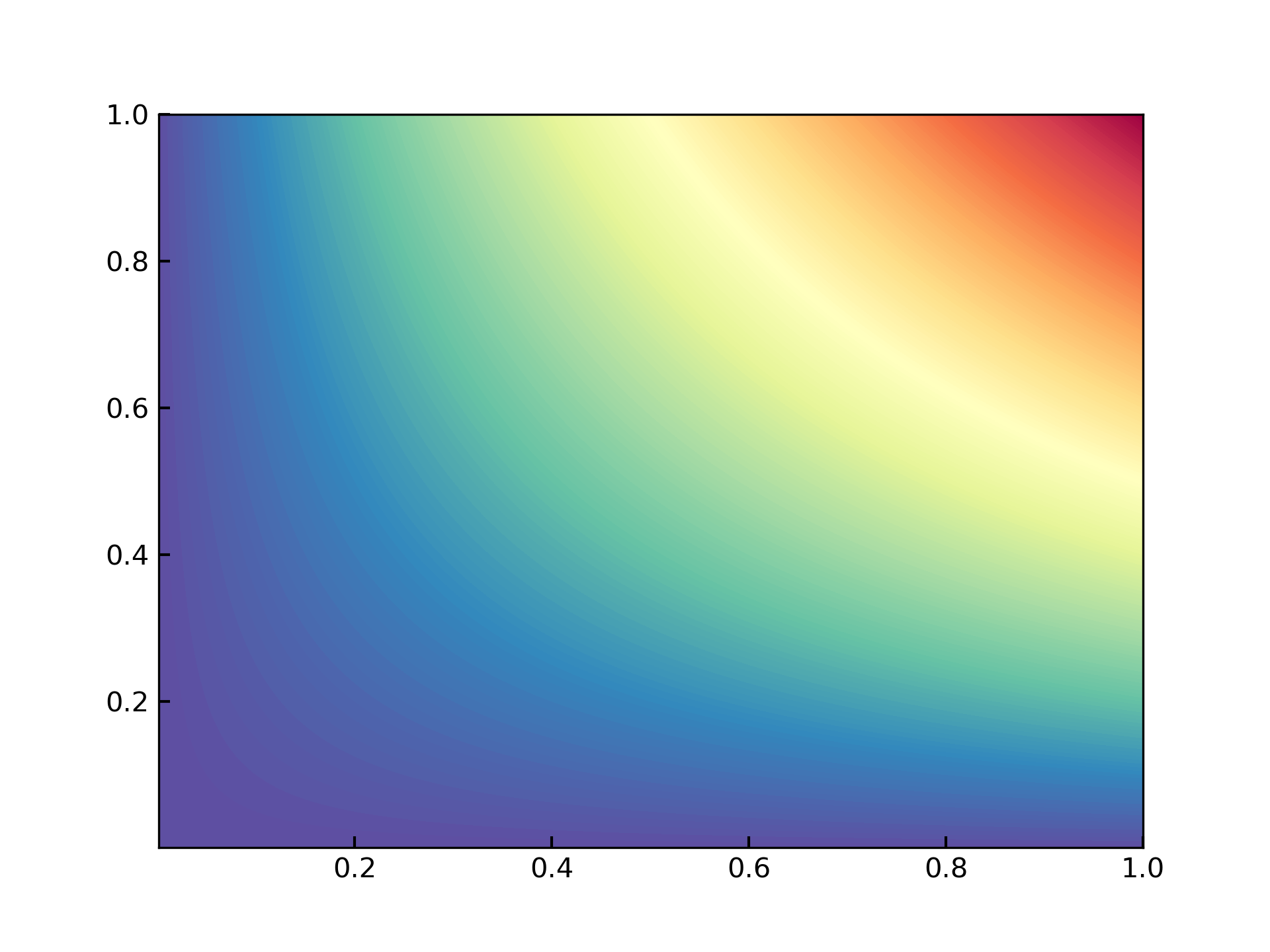} & 
			\includegraphics[width=.32\linewidth]{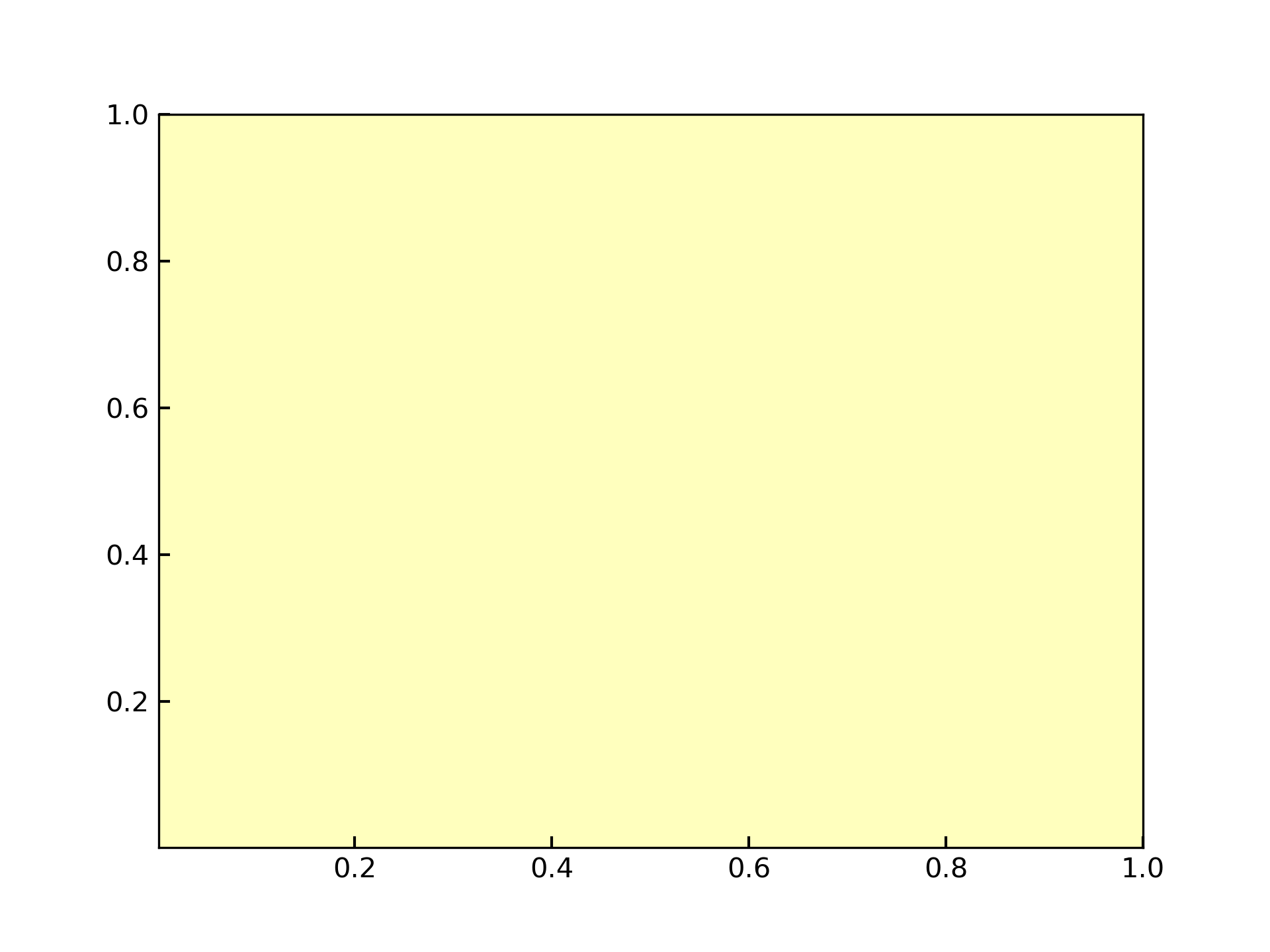}\\
		\end{tabular}
		\caption{Ground Truth}
	\end{subfigure}
	\hfill
	\begin{subfigure}{.45\textwidth}
		\centering
		\begin{tabular}{ccc}
			\includegraphics[width=.32\linewidth]{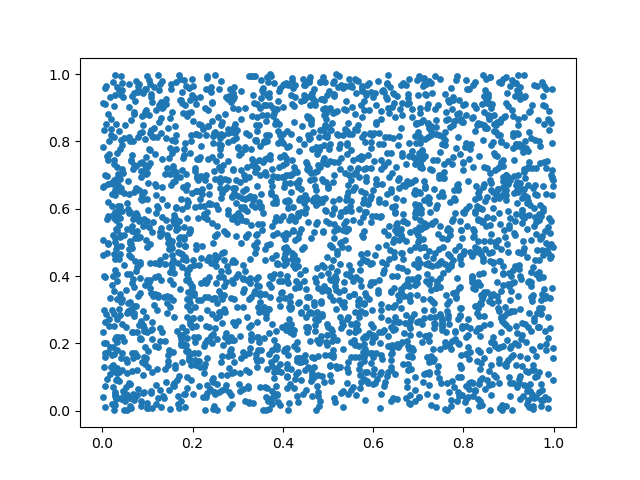} & 
			\includegraphics[width=.32\linewidth]{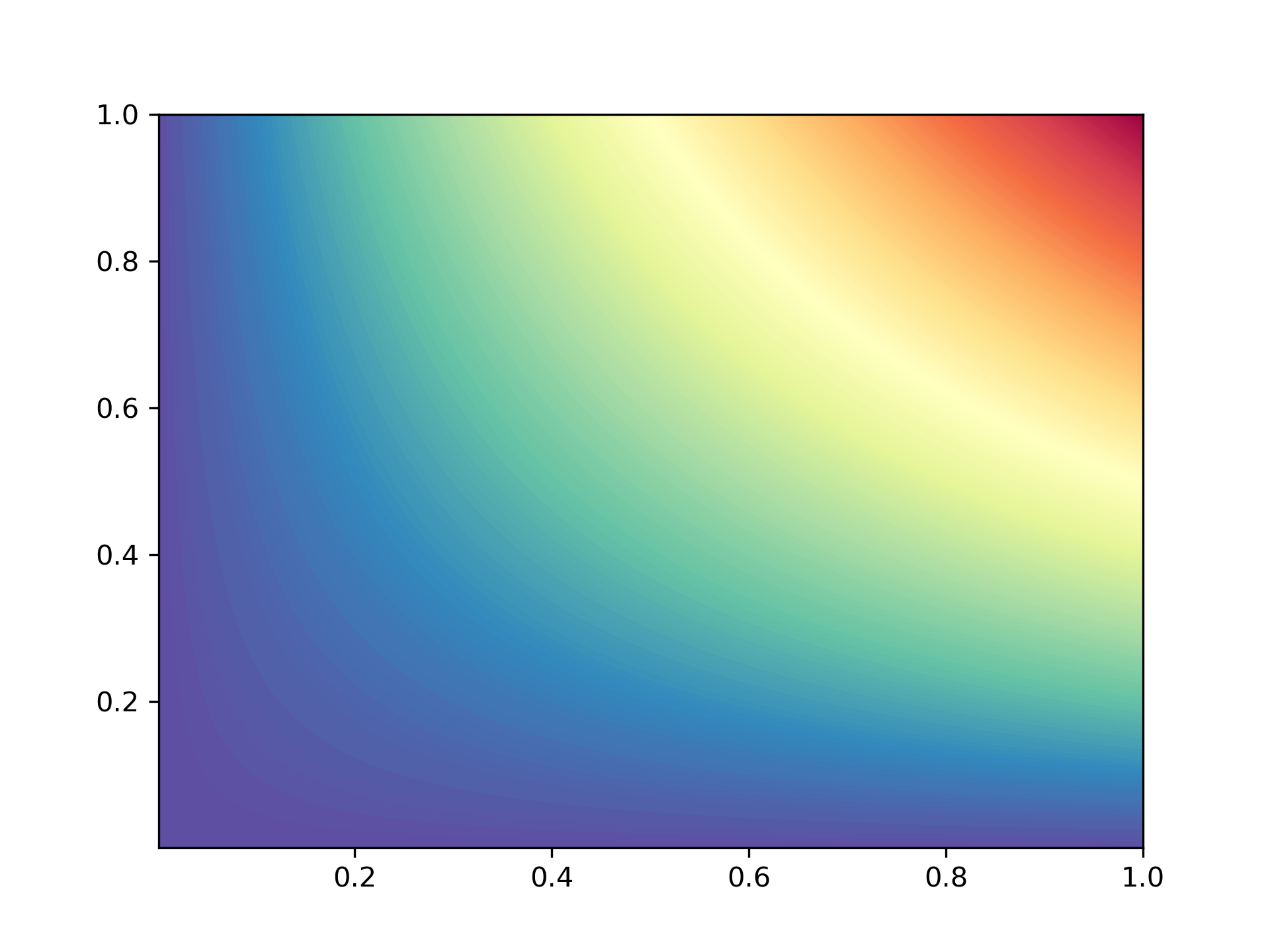} & 
			\includegraphics[width=.32\linewidth]{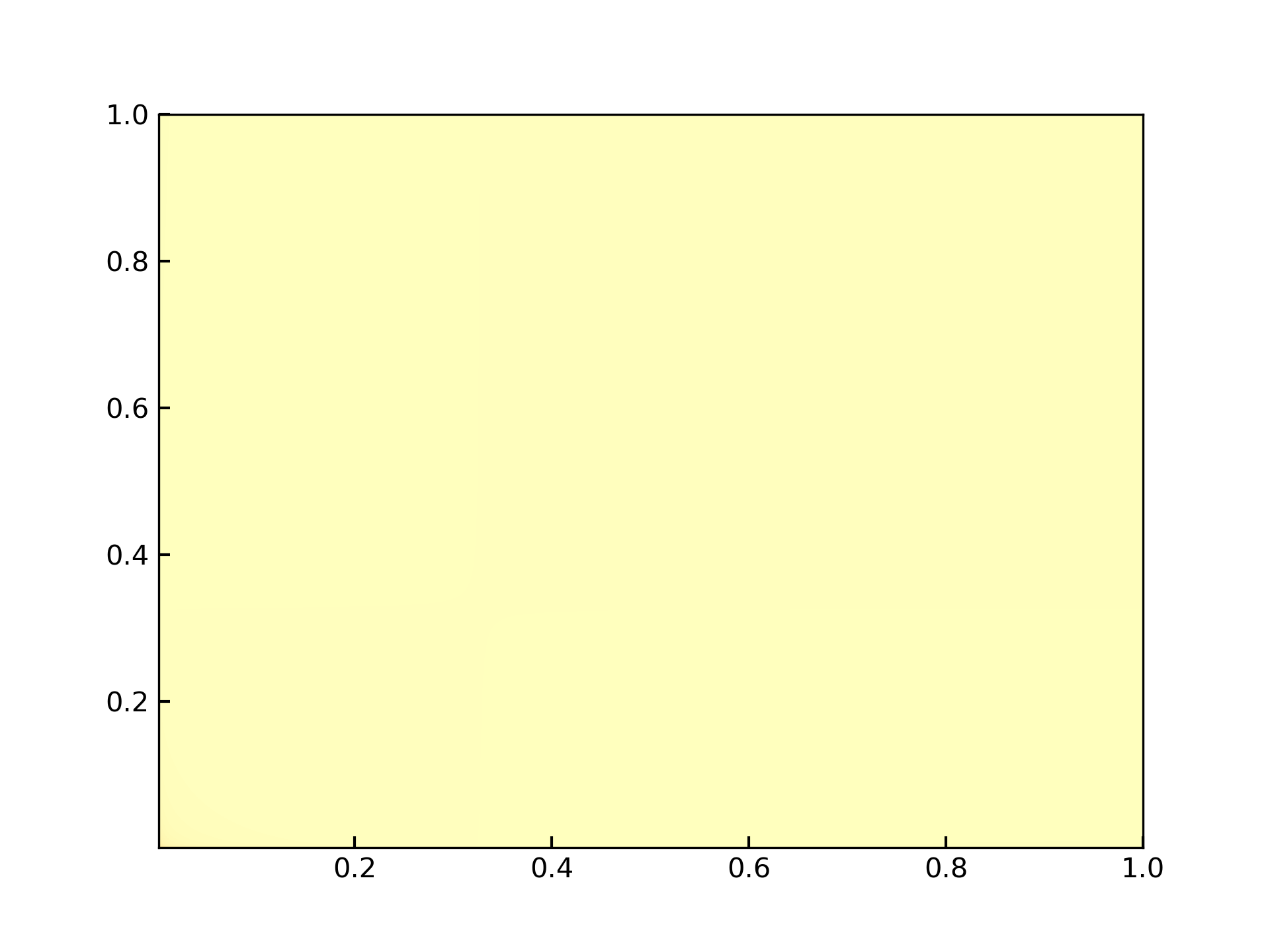} \\
		\end{tabular}
  		\caption{Learned Copula}
	\end{subfigure}	
	\caption{Top to bottom: Learning Clayton, Frank, Gumbel, and Independence copulas using {\ours} from right-censored observations with non-linear survival and censoring risks. Plots from left to right: (i) samples drawn from the ground truth and learned distributions. (ii) joint cumulative distributions, (iii) log probability densities.}
	\label{fig:nonlinear-copula}
\end{figure*}

\end{document}